\documentclass{article}

\usepackage[main,final]{neurips_2026}

\usepackage[utf8]{inputenc} 
\usepackage[T1]{fontenc}    
\usepackage[table]{xcolor}         
\definecolor{lightblue}{HTML}{A7BED3}
\definecolor{blue}{HTML}{3B65CD}
\definecolor{green}{HTML}{81B5AE}
\definecolor{pink}{HTML}{F2545B}
\definecolor{red}{HTML}{A4243B}
\definecolor{lightgray}{HTML}{C3BABA}
\definecolor{darkgray}{HTML}{9A8F97}
\definecolor{darkblue}{HTML}{114083}
\usepackage[hidelinks]{hyperref}
\hypersetup{
    colorlinks=True,
    linkcolor=darkblue,
    urlcolor=darkblue,
    citecolor=darkblue,
    anchorcolor=darkblue}
\usepackage{url}            
\usepackage{booktabs}       
\usepackage{amsfonts,amsmath,amssymb}       
\usepackage{nicefrac}       
\usepackage{microtype}      

\usepackage{fontawesome5}

\newcommand{\legendbox}[1]{%
  \textcolor{#1}{\rule{\fontcharht\font`X}{\fontcharht\font`X}}%
}

\title{Measure Less, Know More: \\ Self-Supervised Test-Time Feature Acquisition}

\usepackage{amsmath,amssymb,amsfonts}
\usepackage{amsthm}
\usepackage{mathrsfs}
\usepackage{bm}
\usepackage{graphicx}
\usepackage{subcaption}
\usepackage{booktabs}
\usepackage{xspace}
\usepackage{hyperref}
\usepackage{etoc}
\usepackage{float}
\usepackage{pifont}
\usepackage{multirow}

\DeclareMathOperator*{\argmin}{arg\,min}
\usepackage{subcaption}
\usepackage[capitalize,noabbrev]{cleveref}
\usepackage{wrapfig}
\usepackage{caption, subcaption}
\usepackage{caption}
\usepackage{multicol}
\usepackage{multirow}
\usepackage{longtable}
\usepackage{booktabs}
\usepackage{adjustbox}
\newtheorem{theorem}{Theorem}[section]
\newtheorem{proposition}[theorem]{Proposition}
\newtheorem{lemma}[theorem]{Lemma}

\newtheorem{assumption}[theorem]{Assumption}

\usepackage[textsize=tiny]{todonotes}

\makeatletter
\renewcommand*{\@fnsymbol}[1]{\ensuremath{\ifcase#1\or \dagger\or \ddagger\or \star\or \S \else\@ctrerr\fi}}
\makeatother

\author{%
  Eeshaan Jain$^{1}$ \quad
  Linus Bleistein$^{1,2}$ \quad
  Bart Deplancke$^{2}$ \quad
  Charlotte Bunne$^{1,2}$ \\[0.5em]
  $^{1}$School of Computer and Communication Sciences, EPFL \\
  $^{2}$School of Life Sciences, EPFL \\
  \texttt{\{eeshaan.jain, linus.bleistein, bart.deplancke, charlotte.bunne\}@epfl.ch}
}

\DeclareRobustCommand\onedot{\futurelet\@let@token\@onedot}
\def\@onedot{\ifx\@let@token.\else.\null\fi\xspace}

\def\RR{\mathbb{R}}
\def\EE{\mathbb{E}}

\def\Prob{\mathbb{P}}
\newcommand{\bad}{\textsc{Bad}}
\def\xhdr{\paragraph}
\usepackage{xspace}
\def\ours{$\textsc{echo}\textrm{-}k$\xspace}
\definecolor{topcornercolor}{HTML}{A32412}
\def\proxyloss{\widetilde{\mathcal{L}}}
\def\trueloss{\mathcal{L}}
\newcommand{\rank}{\textsc{rk}}

\newcommand{\mnist}{MNIST\xspace}
\newcommand{\fmnist}{Fashion-MNIST\xspace}
\newcommand{\minib}{MiniBooNE\xspace}
\newcommand{\hp}{hPancreas\xspace}

\newcommand{\intset}[2]{[\![#1,#2]\!]}

\usepackage{caption}
\usepackage{lipsum}
\usepackage{enumitem}
\newcommand{\oracleset}{\bm{P}^\star}
\newcommand{\ttset}{\widehat{\bm{P}}}

\newcommand{\mb}{\bm{m}}

\newcommand{\Pb}{{\bm{P}}}

\newcommand{\yb}{\bm{y}}
\newcommand{\Acal}{\mathcal{A}}

\newcommand{\Rb}{\bm{R}}
\newcommand{\Yb}{\bm{Y}}
\newcommand{\Eb}{\bm{E}}
\newcommand{\Wb}{\bm{W}}
\newcommand{\Ib}{\bm{I}}

\newcommand{\Cb}{\bm{C}}

\newcommand{\Ical}{\mathcal{I}}

\newcommand{\Ncal}{\mathcal{N}}

\newcommand{\Lcal}{\mathcal{L}}

\newcommand{\Mcal}{\mathcal{M}}
\newcommand{\xb}{\bm{x}}
\newcommand{\Bb}{\bm{B}}
\newcommand{\wb}{\bm{w}}

\newcommand{\zb}{\bm{z}}
\newcommand{\Fcal}{\mathcal{F}}

\definecolor{green1}{rgb}{0.8,1,0.8}
\definecolor{yellow1}{rgb}{1,1,0.87}
\definecolor{blue1}{rgb}{0.9,0.9,1}
\definecolor{red1}{rgb}{1,0.9,0.9}
\newcommand{\first}[1]{{\cellcolor{blue!40} #1}}

\newcommand{\second}[1]{{\cellcolor{pink!40} #1}}

\newcommand{\best}{\first}
\newcommand{\secbest}{\second}

\makeatletter

\expandafter\def\expandafter\normalsize\expandafter{%
    \normalsize%
    \setlength\abovedisplayskip{2pt}%
    \setlength\belowdisplayskip{2pt}%
    \setlength\abovedisplayshortskip{2pt}%
    \setlength\belowdisplayshortskip{2pt}%
}

\begin{document}

\maketitle

\begin{abstract}
\vspace{-10pt}
Recent progress in multimodal, high-dimensional learning has enabled  foundation models to process heterogeneous, large-scale data. However, at test time, acquiring all features or modalities can be prohibitively costly and often redundant. Sequentially selecting  informative modalities is therefore critical, yet challenging when the downstream task or prediction target is unknown.
To this end, we introduce \textsc{echo}-$k$, a \textit{task-agnostic} and \textit{self-supervised} learning principle for modality acquisition: we use a deep model’s \textit{internal} pretrained representations (e.g., from a foundation model) as proxy targets that summarize cross-modal information. We provide theoretical guarantees in a stylized linear setting that motivate a reinforcement learning (RL) policy for sequential  modality selection. 
Across task-agnostic and label-free acquisition baselines, \ours consistently improves budgeted downstream performance across diverse foundation-model backends.
Our method provides a principled route to cost-aware test-time deployment, with implications for any multimodal system where measurements are expensive or time-constrained, and downstream tasks unknown \emph{a priori}. The code is available at \href{https://github.com/bunnelab/echo-k}
{\raisebox{-0.1em}{\faGithub}\,\texttt{github.com/bunnelab/echo-k}}.
\vspace{-10pt}
\end{abstract}

\etocdepthtag.toc{mtmain}
\section{Introduction}

\looseness -1 Multimodal foundation models (FMs) thrive on \emph{more} information: the more modalities we can feed them at test time, the more accurate, robust, and context-aware their representations become \citep{kaplan2020scaling, zhai2022scaling,mizrahi20234m, shukor2025scaling}. In practice, however, not every modality is available by default. Each additional assay, sensor, or data source carries a tangible cost in time, money, throughput, or patient burden. This creates a deployment gap: while models can theoretically produce rich representations by exploiting \emph{multiview} data, acquiring all views for every instance is infeasible.

Modalities are stratified by cost, latency, and availability: in clinical diagnostics, routine measurements are often available immediately, while additional molecular or imaging assays can disambiguate uncertain cases but are constrained by turnaround time, cost, or limited sample material. Similar asymmetries arise in scientific discovery, where inexpensive proxies can be collected at scale, but high-fidelity measurements require specialized instrumentation or scarce experimental time. Measurements are therefore not acquired exhaustively, nor independently, but as structured groups of modalities or features corresponding to coherent sensing processes. We refer to each such group as a \textit{view}. Under a fixed budget, only a small number of views can be queried per instance, and the value of each view depends on which views have already been observed. The acquisition problem is therefore inherently sequential and adaptive: \textit{at test time, given the views observed so far, which view should a foundation model acquire next to most improve its downstream predictions?}

\begin{wrapfigure}{r}{0.5\textwidth}
    \centering
   \includegraphics[width=0.9\linewidth]{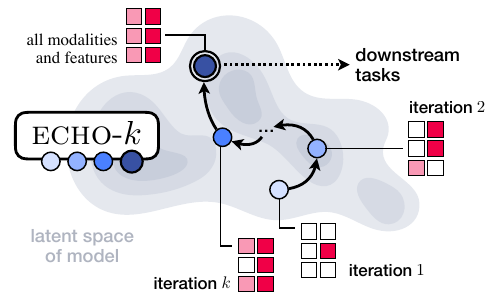}
    \caption{\small \ours learns an RL policy to iteratively selects views to approximate the full latent representation learned by \textit{any} pretrained model, \textit{without} any knowledge of downstream tasks. Once learned, the policy can be deployed for budgeted feature acquisition. }
    \label{fig:main}
\end{wrapfigure}
Existing approaches to sequential acquisition are typically \textit{task-} and \textit{data-}driven: they define utility with respect to a specific downstream objective, such as classification or regression, and optimize expected task loss 
~\citep{gong2019icebreaker,gadgil2023estimating,guney2025active}. 
This paradigm requires labeled data at acquisition time, task-specific reward design, or explicit modeling of task uncertainty. 
As a result, acquisition policies are tightly coupled to a particular task and dataset, and must be retrained or redesigned when downstream objectives change. This reliance on supervision has left the fully task-agnostic, label-free acquisition setting largely unaddressed (see Related Work in \S~\ref{app:related_work}). Yet modern pretrained models are increasingly deployed as frozen backbones whose representations support many downstream tasks.
This shift suggests a different perspective on acquisition: rather than selecting measurements for a single supervised objective, one can ask which measurements are needed to preserve the representation formed by the pretrained model.

In this work, we propose to use these latent representations learned by modern foundation models as a task-agnostic signal for view acquisition. When trained on heterogeneous, multiview data, these models learn compressed representations that integrate complementary evidence across views into a compact state \citep{radford2021learning,simeoni2025dinov3,wenckstern2025ai,oquab2023dinov2,lu2024visual,qu2025tabicl,shmatko2025learning}. 
This latent state encodes the information the model deems relevant across many downstream tasks. 
We use these rich latent states as a learning objective: rather than defining utility with respect to a specific downstream objective, we learn a view selection policy that iteratively selects views to approximate the full latent representation from a minimal number of views.

Learning acquisition policies in representation spaces has two main advantages. On the one hand, it removes dependence on downstream supervision: acquisition decisions can be learned without task labels or task-specific losses. On the other hand, it enables broad transfer: because the representation is shared across tasks, preserving it supports many downstream objectives simultaneously. If the partial-view representation closely matches the full-view representation, then any downstream predictor operating on that representation should incur minimal degradation, independent of the task. 

\looseness -1 We instantiate this principle with \textsc{echo}-$k$, a reinforcement learning algorithm for long-horizon sequential view acquisition that modularly integrates with any self-supervised (foundation) model (Fig.~\ref{fig:main}). At each step, the RL agent selects which view to acquire next 
and the reward quantifies progress toward the corresponding full-view representation. Once trained, the policy is deployed at inference time on unobserved views to allocate acquisition budgets on an instance-by-instance basis. 

We support this algorithm with the first theoretical analysis of self-supervised feature acquisition: in a linear model, where we show that representation recovery controls the latent representation regret (Theorem~\ref{theorem:loss_bound}), and that conditioning on partial observations yields provable gains over fixed orderings. These results in turn justify the RL reward design and motivate the long-horizon acquisition strategy, enabling cost-effective deployment of large multimodal models in settings where full multiview acquisition is impractical.

\textbf{Contributions.} Our contributions are as follows. 
\begin{enumerate}[nosep, leftmargin=*]
    \item We introduce a \textit{model-driven} framework for self-supervised feature acquisition, in which acquisition decisions are guided by the internal representations of pretrained models, without access to task-specific supervision.
        \item We provide the first theoretical analysis of self-supervised feature acquisition. In a stylized linear setting, we derive performance bounds that relate optimal test-time policies and oracle selection strategies, formally characterizing the gap induced by limited information that motivate an RL-based acquisition strategy.
        \item We develop \textsc{echo}-$k$, a cost-efficient reinforcement learning algorithm that learns long-horizon acquisition policies using self-supervised latent representations. The policy is trained to maximize progress toward the full-view representation, translating the representation-recovery principle into a practical sequential acquisition method.

        \item Empirically, we show that \ours learns effective acquisition strategies across diverse data modalities and acquisition horizons. Across image, tabular, single-cell transcriptomics, spatial proteomics, and multimodal geospatial benchmarks, the learned policies rapidly concentrate informative views in early acquisition stages, validating the effectiveness and generality of representation-driven acquisition under realistic budget constraints.

\end{enumerate}
Our contributions bridge theory and practice in unsupervised feature acquisition, enabling foundation models to reason about what to measure next under realistic constraints.

\section{Representation-Driven Acquisition as a Markov Decision Process}
\label{sec:unsupervised_FA}

In the absence of a specified downstream task, acquisition requires a notion of value that is independent of labels. We define this value through the pretrained multiview encoder: acquired views should recover the representation obtained from all views.

\textbf{Setup.} Concretely, we consider $n$ i.i.d. samples, each with $K$ possible views 
$\xb^i = (\mb_1^i, \dots, \mb_K^i)$ drawn from a joint distribution over the product space of multiple views $\Mcal^1 \times \cdots \times \Mcal^K$. \footnote{Throughout, we
use the term \textit{view} to represent either a modality or a feature group acquired jointly, as our framework
applies uniformly to all such settings.} After selection, each view $\mb_k$ is processed by a view-specific encoder $\Phi_k: \Mcal^k \to \RR^{d_k}$ \footnote{Note, under this framework, multiple views can also share the same joint encoder, or correspond to unimodal encoders.} and combined via a 
fixed aggregation $\mathbf{A}_\mathbf{W}: \RR^{d_1} \times \dots \RR^{d_K} \to \RR^p$ parameterized by $\mathbf{W}$, yielding the full-view representation 
$\zb^* = \mathbf{A}_{\mathbf{W}}\!\left(\Phi_1(\mb_1), \dots, \Phi_K(\mb_K)\right)$. This formulation
subsumes both multimodal, where each view corresponds to a distinct sensing modality, and unimodal
models in which views are partitioned into acquisition units but processed by a shared encoder.
For notational simplicity, we absorb encoders into the views and write $\zb^*$ directly. 
We let $\Pb_t \subseteq \intset{1}{K}$ be the acquired set after $t$ steps, with induced 
partial representation $\zb(\Pb_t) := \mathbf{A}_{\mathbf{W}}\!\left(\{\Phi_k(\mb_k): 
k \in \Pb_t\}\right)$.

\xhdr{Objective.} We measure acquisition quality through the representation-recovery loss $
    \Lcal(\Pb_t) := \|\zb(\Pb_t)-\zb^\star\|_2^2,
$
and seek a sequence $(k_1,\ldots,k_T)$ minimizing $\Lcal(\Pb_T)$ under budget $T \ll K$. The ideal one-step oracle greedy rule is
\begin{equation} \label{eqcfeat_aq_problem}
    k_{t+1} \in \argmin_{k \notin \Pb_t}
    \left\|\zb\!\left(\Pb_t \cup \{k\}\right) - \zb^*\right\|_2^2. 
\end{equation}

\xhdr{From training to test-time deployment.} The ideal policy defined above is not deployable at test time. Indeed, it requires both the full-view representation $\zb^\star$ 
and forward passes for every candidate view. During \textit{training}, we use unlabeled data to construct a self-supervised representation-recovery signal to train a view-selection policy. In the complete-reference setting, fully observed training samples allow us to compute the full-view representation $\zb^\star$ and use it as the reference target. More generally, when only a subset of views $R_i\subseteq\intset{1}{K}$ is available for training instance $i$, we can use the corresponding reference representation $\zb(R_i)$ as a partially observed target. In both cases, downstream labels are not used to train the acquisition policy. At \textit{deployment}, the learned policy acts solely from the current 
acquisition state $s_t = (\Pb_t, \zb(\Pb_t), b_t)$; $\zb^\star$ is never accessed.

\xhdr{Downstream tasks.} A representation-recovery objective is only meaningful if minimizing it translates to better performance on downstream tasks. If $g$ is an $L_g$-Lipschitz downstream predictor and the loss $\ell$ is $L_\ell$-Lipschitz in the prediction, then for any acquired set $\Pb$,
\begin{equation*}
    \ell\!\left(g(\zb(\Pb)),\, y\right) - \ell\!\left(g(\zb^*),\, y\right) \;\leq\; L_\ell L_g \,\|\zb(\Pb) - \zb^*\|_2.
\end{equation*}
This provides a simple sufficient condition under which latent recovery bounds the degradation of a fixed downstream predictor. In our experiments, downstream labels are not used to train the acquisition policy; they are used only post hoc to evaluate whether the recovered representations preserve task-relevant information.

The training-deployment asymmetry shown above raises the central theoretical question: how much is lost when acquisition proceeds from partial observations, relative to an oracle greedy policy that can evaluate the full-information value of each view? 
A complete MDP formulation is given in \S~\ref{appendix:extended_math}.

\vspace{-5pt}
\section{Theoretical Guarantees}
\newcommand{\Wm}{\mathbf{W}}
\renewcommand{\wb}{\mathbf{w}}
\label{section:theory}

To quantify the cost of this partial observability, we compare two idealized strategies in a stylized linear setting: (i)~an \emph{oracle policy} that observes the full view vector $\xb$ and greedily selects the view that most reduces representation error, and (ii)~an \emph{optimal test-time policy} that makes the same greedy choice but conditions only on previously acquired views. The gap between such oracle and limited-information strategies has been studied in supervised settings \citep{he2012imitation,madasu2022learning,valancius2023acquisition}; our contribution is to characterize it in the self-supervised, representation-recovery setting, where no task labels are available.

\looseness -1 Let $\xb \in \RR^d$ be a random view vector and $\zb^* = \zb(\xb) := \mathbf{W}\xb \in \RR^m$ a latent representation where $\Wm = (\wb_{ik}) \in \RR^{m \times d}$ is a fixed linear aggregation operator. This setting captures the last-layer behavior of linear or locally linearized representation models. Policies iteratively construct a view subset $\Pb_t \subseteq \intset1d$; at step $t$, the quality of the selected set $\Pb_t$ is measured by the conditional expected square reconstruction error
\begin{align*}
\mathcal{L}(\Pb_t  \mid \Ical_t) &:= \sum_{i=1}^m \EE\!\left[\Big(\sum_{k=1}^d \wb_{ik} x_k - \sum_{k \in \Pb_t} \wb_{ik} x_k\Big)^2 \Big|\; \Ical_t\right] = \sum_{i=1}^m \EE\!\left[\Big(\sum_{k \notin \Pb_t} \wb_{ik}x_k\Big)^2 \Big|\;\Ical_t \right],
\end{align*}
where $\mathcal{I}_t$ denotes the information available to the policy at time $t$, which is different for each policy.

\vspace{-10pt}
\xhdr{One-step oracle policy.} The \emph{one-step oracle policy} greedily selects views with full access to the realization of $\xb$. Given a set $\bm{P}_t^\star$ selected at time $t$, the oracle chooses
\begin{align*}
k_t^\star &\in \argmin_{k \notin \bm{P}_t^\star}
\sum_{i=1}^m \Bigg(
\sum_{j \notin \bm{P}_t^\star \cup \{k\}} \wb_{ij} x_j
\Bigg)^2 = \argmin_{k \notin \bm{P}_t^\star} \sum_{j \notin \Pb_t^* \cup \{k\}} \underbrace{\|\wb_j\|_2^2x_j^2}_{ =: \Eb_j} + \sum_{\substack{j \neq j' \\ j,j' \notin \Pb_t* \cup \{k\}} }\underbrace{\wb_j^\top \wb_{j'} x_j x_{j'}}_{=: \Cb_{jj'}},
\label{eq:oracle_rule}
\end{align*}
where $\wb_k$ denotes the $k$-th column of $\Wm$. The oracle therefore exploits both individual view contributions $\Eb_k$ (energies) and cross-view interaction terms $C_{kk'}$, made possible by full knowledge of $\xb$. A test-time policy, by contrast, cannot evaluate the $C_{kk'}$ terms, since they depend on the unobserved views; it must instead rely on conditional expectations given observed views.

\vspace{-8pt}
\xhdr{Test-time policies.} We define the class of \emph{one-step test-time policies} as stochastic mappings that, at each step, select a new view based solely on the history of previously observed views and update their beliefs over the unobserved ones. Formally, such policies are adapted to the filtration $\Fcal_t := \sigma(x_k : k \in \Pb_t)$ and therefore operate under partial observability, in contrast to the oracle. 
Among this class, the \emph{optimal test-time policy} greedily minimizes the conditional expected representation loss, with perfect knowledge of the conditional distribution of unobserved views given observed ones. Given the current selected set $\widehat{\bm{P}}_t$, this policy chooses
\begin{align*}
\widehat{k}_t \in \argmin_{k \notin \widehat{\Pb}_t} \Bigg[\sum_{j \notin \widehat{\Pb}_t \cup \{k\}} \EE[\Eb_j | \Fcal_t] + \sum_{\substack{j \neq j' \\ j,j' \notin \widehat{\Pb}_t \cup \{k\}}} \EE[\Cb_{jj'} | \Fcal_t]\Bigg],
\end{align*}
where $\Eb_j$ and $\Cb_{jj'}$ are as defined in the oracle decomposition. 
To quantify performance, we define the \emph{true} loss incurred after selecting $\Pb_t$ as $\trueloss(\Pb_t,\xb) := \left\|\sum_{k\notin \Pb_t}\wb_k x_k\right\|_2^2$,
and the latent regret of the test-time policy as 
$
    \mathcal{R}_t := \mathbb E _{\xb}\Big[\trueloss(\ttset_t,\xb)\Big] - \mathbb E _{\xb}\Big[\trueloss(\oracleset_t,\xb)\Big],
$
where $\bm{P}_t^\star$ denotes the view set selected by the oracle after $t$ steps. 

\xhdr{Proxy loss.} 
The true loss includes pairwise interaction terms $\Cb_{jj'}$, so the marginal value of a view depends on the entire set of unobserved views; thus bounding expected regret would  require controlling the distribution of orderings in a quadratically coupled system, which is analytically intractable without additional structure. To isolate the effect of partial observability from this combinatorial difficulty, we analyze a proxy objective $\proxyloss(\bm{P}_t, \xb) := \sum_{k \notin P_t} \mathbb E [\Eb_k\mid \mathcal{I}_t]$ and connect it back to the true regret in Theorem~\ref{theorem:loss_bound}.
Two quantities govern the proxy analysis: the centered variables $\Yb_{j,s} := \Eb_j - \EE[\Eb_j | \Fcal_s]$ measure how far each view's energy deviates from its conditional expectation at step $s$, and the maximal conditional deviation $\Rb_s := \max_{i,j \notin \widehat{\Pb}_{t}} (Y_{i,s} - Y_{j,s})$ captures the worst-case spread of these estimation errors among unselected views.

\subsection{Pathwise Decomposition of the Proxy Regret} \label{sec:pathwise}
We now study the gap between the one-step oracle policy and the optimal test-time policy under the proxy objective. Lemma~\ref{lemma:diff_proxy_loss} provides a pathwise bound on their difference.
\begin{lemma}
\label{lemma:diff_proxy_loss}
For all $t \ge 2$, the proxy loss gap satisfies
\begin{equation}\proxyloss(\ttset_t,\xb) - \proxyloss(\oracleset_t,\xb) \leq \sum\limits_{s=1}^{t} \bm{1}_{\{\widehat{k}_s \in B_t\}} \Rb_s,
\end{equation} where $B_t$ denotes the mismatch set between the oracle and test-time selections up to time $t$.
\end{lemma}
Lemma~\ref{lemma:diff_proxy_loss} reduces the proxy regret to a sum of weighted \emph{mistake events}: regret is incurred only when the test-time policy selects a view outside the oracle-selected set, and the magnitude of each mistake is controlled by the conditional deviation $\Rb_s$. 

Crucially, extending this bound beyond a pathwise statement requires detailed control over the distribution of the order statistics of $(\Eb_k)_{k=1}^d$. In particular, bounding the expected regret entails enumerating all possible orderings of the view energies and characterizing the frequencies and magnitude of deviation of the two policies. Without structural assumptions on the joint distribution of $(\Eb_k)$, we find such a control to be intractable.
We therefore posit a latent-variable model in which the data is driven by a shared discrete factor with Gumbel noise, which induces a tractable Plackett–Luce distribution over view rankings \citep{yellott1977relationship,ma2021learning} and makes the regret analysis possible.

\begin{assumption}[Gumbel latent correlation model]
\label{assumption:dgp}
The view vector $\xb \in \mathbb{R}^d$ is strictly positive and satisfies
$\log\xb = \bm{\eta} + \bm{\beta}\odot Z + \varepsilon,
$
where $\bm{\eta}, \bm{\beta} \in \mathbb{R}^d$ are deterministic, and the common latent factor $Z$ is a discrete latent variable taking values in ${z_1,\dots,z_M}$ with probabilities ${p_1,\dots,p_M}$. The noise terms follow $\varepsilon_j \overset{\mathrm{i.i.d.}}{\sim}\mathrm{Gumbel}(0,\sigma)$ with $\sigma \in (0,1/4)$. We further assume that for all latent states $z_m$, $\mathbb{P}[Z=z_m \mid \mathcal{F}_s] > 0$.
\end{assumption}

Under Assumption~\ref{assumption:dgp}, the conditional ordering of the energies $(\Eb_j)_{j=1}^d$ given $Z=z_m$ follows a Plackett–Luce distribution. Specifically, for any permutation $\pi$ of ${1,\dots,d}$,
\begin{align*}
    \Prob \Big(\Eb_{\pi(1)} \geq \dots \geq \Eb_{\pi(d)}\mid Z = z_m\Big) = \prod\limits_{s=1}^d \frac{e^{h_{\pi(s),z_m}}}{\sum\limits_{r=s}^d e^{h_{\pi(r),z_m}}},
\end{align*}
where $(h_{j,z_m})_{j,m}$ are model-dependent utilities determined by $(\bm{\eta},\bm{\beta},\sigma)$. Marginally, the induced ranking distribution is therefore a mixture of Plackett–Luce models over the latent states of $Z$. 
The role of the Gumbel assumption is not primarily to model the data distribution itself, but to ensure that the ranking of view energies follows a tractable parametric family, which is what makes the regret analysis possible. Fig.~\ref{fig:app-cor} provides visualizations of datasets generated under this model.

\subsection{Regret Bounds under the Gumbel Model} \label{sec:regret_bounds}
For each latent state $m$, define the utilities $u_j^{(m)} := 2\log (\| \wb_j\|_2) + 2\eta_j + 2\beta_j z_m$ and let $\gamma := \min_{m \in [M]} \min_{j \neq j'} \left| u_j^{(m)} - u_{j'}^{(m)} \right| > 0$ denote the minimum separation between utilities. A larger $\gamma$ allows easier identification of view-energy rankings. Also, denote  $p_{\min} = \min_{m=1,\dots, M}p_m$.

A key step is to bound the number of disagreements between the oracle and optimal test-time policies. Let $G_t$ (resp.\ $B_t$) denote the set of views selected only by the oracle (resp.\ optimal test-time) policy at time $t$, and define
\[
R_{\max} := \max_{j \in [d]} \frac{\max_m e^{2\beta_j z_m}}{\min_m e^{2\beta_j z_m}} < \infty, \text{ and } \varepsilon_\gamma := \frac{e^\gamma - 1}{R_{\max} + e^\gamma - 1} \in (0, 1).
\]
\begin{lemma}
The expected number of oracle-test-time mismatches satisfies, for any $\lambda > 0$,
\begin{equation}
\EE|B_t| = \EE|G_t| \leq D_t(\lambda) := \min\left\{t, d-t, \frac{C_\lambda}{1-\rho_\lambda} + \frac{1}{e^{\rho_*}-1}\right\}
\end{equation}
where $C_\lambda := (M - 1) \left( \frac{1 - \varepsilon_\gamma}{\varepsilon_\gamma p_{\text{min}}} \right)^\lambda$, $\rho_\star := 1-e^{-\gamma/(2\sigma)}$, and $\rho_\lambda$ is a contraction constant (see \S~\ref{technical_lemma:posterior-concentration}).
\end{lemma}

\paragraph{Proof outline.}
The analysis proceeds in three steps:
(i) we bound the number of oracle--test-time disagreements
(ii) this yields a proxy regret bound via the pathwise decomposition;
(iii) we relate proxy and true losses through a structural assumption on the embedding weights. For proofs, see \S~\ref{appendix:proofs}.

The regret on the proxy controls errors due to partial observability, but says nothing about the gap between the proxy and true losses. Bridging this gap requires a condition on the geometry of the embedding weights: if the normalized columns of $\Wm$ are approximately orthogonal, the cross-view interaction terms $\Cb_{jj'}$, that the proxy ignores, remain small.

\begin{assumption}[Near Isometry]
\label{asm:isometry}
     There exists constants $C$, $p_{\textnormal{eff}} > 0$ such that for any subset $J \subseteq \{1,\dots,d\}$ of size $m$, the matrix $\Bb_J = \big(\wb_u /\left\|\wb_u\right\|\big)_{u \in J}$ satisfies
    \begin{align*}
        \left\|\Bb_J^\top \Bb_J - \Ib_m\right\|_\textnormal{op} \leq C \Bigg(\sqrt{\frac{m}{p_\textnormal{eff}}} + \frac{m}{p_{\textnormal{eff}}}\Bigg). 
    \end{align*}
\end{assumption}

This condition enforces approximate orthogonality of embedding directions and is empirically satisfied by large-scale vision transformers such as the DINOv2 family~\citep{oquab2023dinov2} and by trained autoencoders across a range of dimensions (\S~\ref{appendix:viz_weights}). It ensures that the cross-view interaction terms ignored by the proxy objective remain controlled, so that the proxy loss remains aligned with the true latent reconstruction loss. We can now state the main result.
\begin{theorem}
\label{theorem:loss_bound}
    Take $p \in \big(1,\frac{1}{2\sigma}\big)$ and let $q = \frac{p}{p-1}$. Under Assumptions \ref{assumption:dgp} and \ref{asm:isometry}, for any $\lambda >0$, the latent regret satisfies
    \begin{align}
        \mathcal{R}_t \leq& \left(\sum_{i=1}^d M_{(i)} \min(t, d-i+1)\right) ^{1/p} D_t^{1/q} 
        +  C' \Bigg(\sqrt{\frac{N_t}{p_{\textnormal{eff}}}} + \frac{N_t}{p_{\textnormal{eff}}}\Bigg)  \sqrt{D_tN_t},
\end{align}
where $M_j := \EE[\Eb_j^p]$ and $M_{(1)} \ge \cdots \ge M_{(d)}$ be the sorted $p$-th moments, and $N_t = d -t + D_t$.
\end{theorem}

\begin{wrapfigure}{r}{0.4\textwidth}
    \vspace{-15pt}
    \centering
   \includegraphics[width=0.8\linewidth]{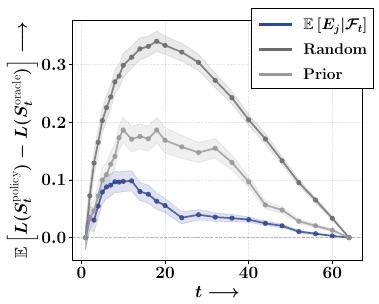}
    \caption{\small Expected true loss gap curve between the oracle and the test-time optimal, random, and prior-based ($\mathbb{E}[\Eb_j]$) policies.}
    \label{fig:croissant}
    \vspace{-30pt}
\end{wrapfigure}
The bound separates two sources of error. 
The first term is an \emph{oracle-mismatch term}:
it captures the cost of partial observability through the disagreement envelope $D_t$. 
This term decreases as the utility gap $\gamma$ increases, reflecting that well-separated energy rankings are easier to infer from partial observations. 
The second term is a \emph{spectral misalignment term}: it controls the error introduced by replacing the true latent loss with the marginal-energy proxy. Consequently, the bound is boundary-correct: since $D_t=0$ at $t=0$ and $t=d$, both terms vanish at the endpoints, matching the reverse-U regret profile observed empirically in Fig.~\ref{fig:croissant}.

\def\trueloss{\mathcal{L}}

\section{Methodology}
\label{sec:method}

\subsection{From Theoretical Test-Time Policies to Learnable Policies}

The theoretical analysis in \S~\ref{section:theory} yields three insights that directly inform algorithm design. First, Theorem~\ref{theorem:loss_bound} 
shows that representation-recovery errors control latent regret up to an oracle-mismatch term and a spectral misalignment term, which becomes small when embedding directions are approximately orthogonal. 
This motivates using progress toward the full-view representation as a self-supervised acquisition signal. Second, the posterior concentration result (Lemma~\ref{technical_lemma:posterior-concentration}) shows that conditioning on observed views can reduce the gap to a full-information oracle, motivating adaptive, instance-specific policies that update after each acquisition. Third, the analysis is restricted to \emph{myopic} policies: the oracle itself is one-step greedy, and the regret bound characterizes the gap between two greedy strategies. In practice, the conditional law of unobserved views given $\Fcal_t$ is unknown, and acquisition is a finite-horizon decision problem in which early choices affect the value of later measurements, especially when views are correlated \citep{shim2018joint,zannone2019odin,kachuee2019opportunistic,schutz2025afabench}. 
This motivates moving beyond the greedy regime to long-horizon planning.


Together, these insights culminate in \ours, a reinforcement-learning algorithm for self-supervised, long-horizon view acquisition. We use the theory to define a reward, and learn a parameterized policy that optimizes cumulative improvement in representation quality over an entire acquisition trajectory, rather than approximating the one-step oracle.
Formally, we aim to learn a parameterized policy $\widehat{\pi}_{\bm{\theta}}$ that maximizes cumulative discounted improvement starting from any step $s$
\begin{align*}
    \max\limits_{\bm{\theta} \in \bm{\Theta}} \,\,\sum\limits_{t=s}^d \alpha^{t-s} \EE \big[R_t\Big(\pi_{\bm{\theta}}(\bm{P}_{t-1})\big)\Big],
\end{align*}
where $R_t$ denotes the incremental improvement in latent representation quality at step $t$, $\alpha \in (0,1]$ is a discount factor that allows the policy to trade off immediate gains against future benefits. The self-supervised reward requires a reference representation against which partial representations are compared. We consider two training regimes. In the \emph{complete-reference} setting, each training instance is fully observed, so the reference is the full-view representation $\zb_i^{\star}=\zb_i(\intset{1}{K})$. In the \emph{partially observed-reference} setting, training instance $i$ contains only a subset of available views $R_i\subseteq \intset{1}{K}$; in this case, we use $\zb_i^{\mathrm{ref}}=\zb_i(R_i)$ as a noisy but label-free reference target. 
\subsection{Learning Representation-Driven Policies}
\looseness -1 
\xhdr{Reward.} 
Given the reference representation $\zb_i^{\mathrm{ref}}$, we define the representation error for instance $i$ as $\ell_i(S) := \|\zb_i^{\mathrm{ref}}-\zb_i(S)\|_2 $. 
The step-wise reward is the normalized improvement in the error,
\begin{equation}
\label{eq:echok_reward}R_t(\widehat{\pi}_{\bm{\theta}})
:=
\omega_t
\frac{1}{N}
\sum_{i=1}^N
\frac{
\ell_i(\Pb_{t-1}^i)
-
\ell_i(\Pb_{t-1}^i\cup\{k_t^i\})
}{
\ell_i(\varnothing)+\epsilon
},
\end{equation}

where $N$ is the minibatch size, and $k_t^i\sim \widehat{\pi}_{\bm{\theta}}(\cdot\mid s_{t-1}^i)$ is the view selected for instance $i$ at step $t$. We use the unsquared norm rather than the squared loss analyzed in \S~\ref{section:theory}, as this provides more stable reward scaling in practice. The reward~\ref{eq:echok_reward} retains the same structure as the representation-recovery objective studied theoretically: it measures marginal progress toward a reference latent representation without using downstream labels. The factor $\omega_t$ optionally upweights early acquisition steps, where marginal gains are typically largest, and normalization prevents high-error instances from dominating the learning signal. Moreover, when views have nonuniform acquisition costs $c_k$, we use a cost-aware reward of the form $r_t^{\mathrm{cost}} = r_t - \lambda_c \frac{1}{N} \sum_{i=1}^N c_{k_t^i}$ with costs normalized across views.

\xhdr{Policy input.} At each step, the policy receives the current partial representation $\zb_i(\Pb_t^i)$, the observed view values with unobserved entries masked, the acquisition mask, and the remaining budget. In heterogeneous-cost experiments, the state is additionally augmented with the view-cost vector and remaining cost budget; the full state specification is given in \S~\ref{appendix:implementation}.  

\paragraph{Optimization.}
We optimize $\widehat{\pi}_{\bm{\theta}}$ using Proximal Policy Optimization (PPO) \citep{schulman2017proximal}, with invalid action masking~\citep{huang2020closer} to prevent re-selection of already acquired views. During both training and deployment, the policy has access only to the frozen representation model and partial observations induced by the current acquisition set.

\section{Experimental Results}
\label{section:exp}
We evaluate \ours across controlled synthetic experiments, vision, and tabular encoders, and three frozen foundation-model settings: single-cell transcriptomics with TABULA~\citep{dingtabula}, spatial proteomics with VirTues~\citep{wenckstern2025ai}, and multimodal geospatial sensing with a self-supervised multimodal encoder~\citep{gordon2026mmearth}. 
The experiments are designed to answer four questions: (i) whether representation-driven acquisition improves downstream performance under limited budgets; (ii) whether the same label-free acquisition principle composes with heterogeneous pretrained backbones; (iii) whether learned acquisition policies transfer to downstream tasks unseen during policy learning; and (iv) whether the method remains effective under realistic deployment constraints, including partially observed reference data and heterogeneous acquisition costs.
We discuss the results for foundation models, grouped by application domain, here. 
Additional controlled benchmarks on synthetic data, and vision datasets such as MNIST, Fashion-MNIST, and MiniBooNE, a physics benchmark, along with incorporating costs and adaptive stopping, are shown in \S~\ref{appendix:additional_experiments}. 

\xhdr{Evaluation.} Across all experiments, the representation model is frozen and receives partially observed inputs. \ours learns an acquisition policy using only self-supervised rewards derived from progress toward a reference latent representation. In the complete-reference setting, this reference is the full-view embedding; in the partially observed-reference setting, it is computed from the subset of views available for each training instance. Downstream labels are never used for policy learning and are reserved exclusively for post-hoc evaluation. At deployment, the policy observes only the values of the views acquired so far and the acquisition mask, and never accesses the full-view representation or downstream labels. To measure downstream performance, we train a single linear probe per downstream task on partial embeddings using random masking augmentation, and keep it fixed across acquisition budgets. The same probe is applied to the partial embeddings induced by every acquisition method, so that differences in performance reflect only which views were acquired. Moreover, a single acquisition policy is learned per backbone and reused across all budgets and downstream tasks, including tasks that were not specified during policy learning. All results are averaged over three random seeds (\S~\ref{appendix:evaluation}).

\xhdr{Baselines and metrics.} We compare against 12~baselines spanning oracle/random references, model-driven heuristics (\texttt{GIR}, ESM2-based orderings), and unsupervised feature selection methods (\texttt{Laplacian Score}, \textsc{spec}-$\varphi_3$, \texttt{MCFS}, \texttt{UDFS}, \texttt{SVD-Ent}, \texttt{Variance}) adapted as acquisition orders (see \S~\ref{appendix:baselines} for details). The \texttt{Oracle} greedily selects the view that maximally reduces the distance to the full-view embedding using unobserved values; it is not deployable at test time and serves as an upper bound on greedy acquisition. Domain-specific baselines are evaluated only where they are defined, e.g., ESM2-based orderings require protein-sequence embeddings and are therefore restricted to protein markers. Since most of these baselines induce a fixed, instance-independent acquisition order, we additionally compare against label-free instance-adaptive policies (\S~\ref{subsec:adaptive}) and supervised acquisition methods that use downstream labels (\S~\ref{appendix:supervised}). For regression, we report $R^2$ or MSE, and for other real-world tasks, classification accuracy as a function of the acquisition budget $t \in \intset{1}{d}$, reported as a fraction of the total number of views. To directly assess the training objective, we further report the normalized representation recovery (\S~\ref{subsec:recovery}). For full evaluation details, refer to \S~\ref{appendix:exp_details}.

\begin{figure*}[t]
\centering
\includegraphics[width=\linewidth]{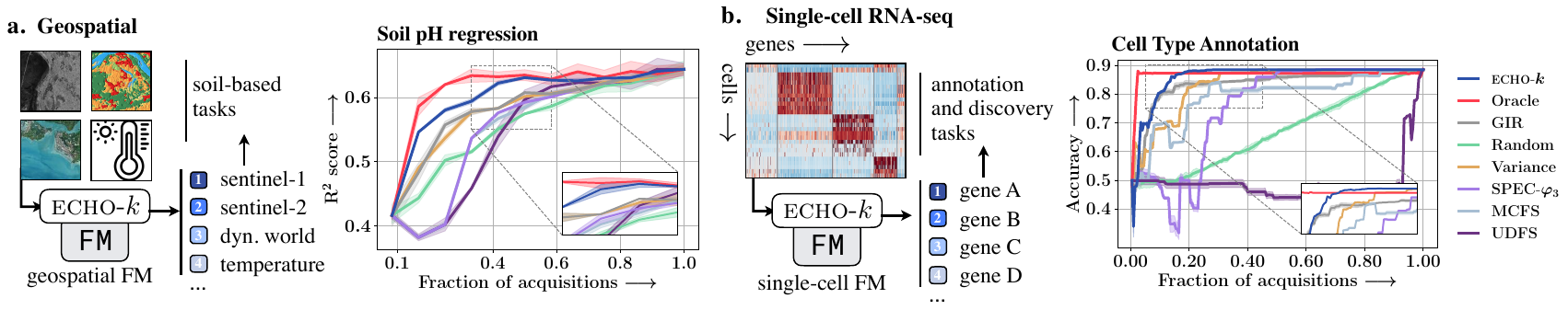}
\caption{ \textbf{a.} Results of \ours on acquiring geospatial modalities to predict soil pH regression on MMEarth-Bench ($d=12$) \textbf{b.} \ours can operate on single-cell RNA-seq foundation models to sequentially acquire informative genes that support broad set of downstream tasks. We further present comparison of \ours to baselines, evaluated on a cell type annotation task on \hp (with $d=200$).  }
\label{fig:scrnaseq}
\end{figure*}

\subsection{Evaluation on Foundation Model Backends}
\xhdr{Single-cell foundation models. }
\looseness -1 In single-cell RNA-seq, 
selecting the \emph{right} genes is a critical bottleneck, and the same panel must often support multiple analyses such as cell type annotation, trajectory inference, or gene program discovery that may not be specified at design time (Fig.~\ref{fig:scrnaseq}a). We pair \ours with TABULA~\citep{dingtabula}, a FM for single-cell transcriptomics (\S~\ref{appendix:architectures}).

\looseness -1 \emph{Results.}
We evaluate on \hp, a curated single-cell transcriptomics (RNA-seq) collection from human pancreatic tissue~\citep{dingtabula}, retaining the top 200 highly variable genes across 14 cell types. We use cell type annotation as a representative downstream task; it is the standard evaluation objective in single-cell analysis and requires the gene subset to preserve fine-grained phenotypic distinctions (\S~\ref{appendix:datasets}). \ours underperforms slightly during the first $\sim$2--3\% of acquisitions but rapidly recovers and consistently outperforms \texttt{Variance}, \texttt{GIR}, and \texttt{Random} thereafter (Fig.~\ref{fig:scrnaseq}b). \texttt{Random} performs poorly throughout, reflecting the sparsity of informative genes in high-dimensional expression data.

\begin{figure*}[t]
  \includegraphics[width=\textwidth]{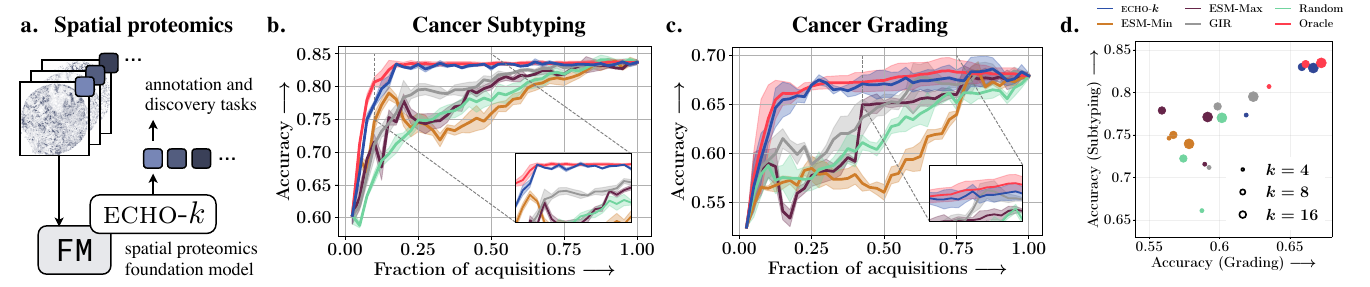}
  \caption{\textbf{Results of \ours\ on the NSCLC dataset. }\textbf{a.} Overview of the experimental setting: given high-dimensional multiplex imaging data, we study how foundation models trained on such data can be leveraged to sequentially acquire informative markers for downstream annotation and discovery tasks.
\textbf{b.} Comparison of \ours\ with baseline methods on the cancer subtyping task.
\textbf{c.} Comparison of \ours\ with baseline methods on the cancer grading task.
\textbf{d.} Scatter plot illustrating the accuracy trade-off between cancer subtyping and grading when using the same set of views.}
  \label{fig:cords}
\end{figure*}

\xhdr{Spatial proteomics foundation models. }
\label{subsec:spatial}

In spatial proteomics, multiplexed protein imaging characterizes tissue through successive rounds of antibody staining, each revealing a new marker and adding to assay cost, imaging time, and tissue degradation risk. Panel design must therefore support multiple diagnostic objectives simultaneously (e.g., cancer subtyping \emph{and} grading), with different patients potentially benefiting from different markers (Fig.~\ref{fig:cords}a). We compose \ours with the VirTues foundation model~\citep{wenckstern2025ai}, which embeds multiplex imaging data into tissue-level representations. At deployment, the policy mirrors the iterative staining workflow: selecting which marker to acquire next, updating its partial representation, and using the updated embedding to inform the next selection (\S~\ref{appendix:implementation}).

\emph{Results.}
We evaluate on a non-small cell lung cancer (NSCLC) patient cohort~\citep{cords2024cancer} profiled by imaging mass cytometry with 40 protein markers spanning tumor, immune, stromal, and vascular compartments. We consider two clinically relevant downstream tasks: cancer subtyping (2 classes) and cancer grading (3 classes), the latter being substantially more challenging (\S~\ref{appendix:datasets}). \ours consistently outperforms baselines on both tasks (Figs.~\ref{fig:cords}b,c). The scatter plot in Fig.~\ref{fig:cords}d shows that \ours achieves favorable accuracy on \emph{both} tasks under identical acquisition orders, demonstrating task-agnostic panel design. A qualitative analysis (\S~\ref{appendix:marker_analysis}) reveals that \ours prioritizes clinically established markers such as \texttt{panCK}, \texttt{Ki-67}, and  \texttt{CD4} early in the process, consistent with domain knowledge~\citep{martin2004ki,su2006role}, despite receiving no supervision. Acquisition orders differ substantially across samples, with mean pairwise Kendall’s $\tau=0.06$. Across training seeds, orders for the same sample show greater agreement: mean $\tau=0.512$ and 61.6\% top-five overlap, compared with 12.8\% under random ordering.

\xhdr{Multimodal geospatial foundation models. } 
Earth observation integrates heterogeneous sensing products, from optical and radar imagery to terrain, land-cover, canopy-height, meteorological, and geolocation-derived variables. We evaluate \ours on MMEarth-Bench~\citep{gordon2026mmearth}, a globally distributed benchmark of 12 aligned geospatial modalities, treating each modality as a separately acquirable view.

\emph{Results.}
We evaluate on soil pH regression as an unseen downstream task, reporting $R^2$
 as a function of the number of modalities. \ours outperforms non-oracle baselines consistently under limited modality budgets, reaching strong predictive performance after acquiring only 20-30\% of the views (Fig.~\ref{fig:scrnaseq}b).  We also notice that \texttt{GIR} and Variance-based baselines are the second-best performers.

\subsection{Robustness to Partially Observed Reference Data}
In Table~\ref{tab:incomplete_data}, we evaluate whether \ours remains effective when trained from incomplete views by randomly masking a fraction $m$ of views before computing the reference embedding, and train the policy against this partially observed target. For single-cell transcriptomics FMs, performance is nearly unchanged for $m=$10\%, remains strong for $m=$20\%, and still recovers high accuracy under heavier missingness once moderate acquisition budgets are allowed. For proteomics FMs, performance degrades more gradually with $m$, but partially observed-reference policies remain competitive with strong baselines across budgets. We also notice that for high budgets, \texttt{GIR} is competitive with \ours. Overall, \ours degrades smoothly rather than collapsing, suggesting that useful acquisition policies can be learned from incomplete reference representations.

\begin{table}[t]
\centering
\setlength{\aboverulesep}{0pt}
\setlength{\belowrulesep}{0pt}
\caption{\small \textbf{Results of robustness to partially observed reference data.} Performance of \ours when an m-fraction of views is randomly masked while computing the training reference representation. Accuracy is reported at different acquisition budgets on hPancreas and NSCLC grading. Lower $m$ corresponds to more complete reference data. \legendbox{blue!40} denotes best and \legendbox{pink!40} second best performance, apart from the Oracle.}
\label{tab:results}
\adjustbox{max width=\textwidth}{
\begin{tabular}{lccccc@{\hspace{\tabcolsep}\hspace{\tabcolsep}}lcccc}
\toprule
\multirow{2}{*}{\textbf{Method}} & \multicolumn{4}{c}{\textbf{hPancreas}} &  & \multirow{2}{*}{\textbf{Method}} & \multicolumn{4}{c}{\textbf{NSCLC (Grading)}} \\
\arrayrulecolor{lightgray} \cmidrule{2-5} \cmidrule{8-11} \arrayrulecolor{black}
 & \textbf{20\%} & \textbf{40\%} & \textbf{60\%} & \textbf{80\%} &  &  & \textbf{20\%} & \textbf{40\%} & \textbf{60\%} & \textbf{80\%} \\ \midrule
Oracle & $0.872 {\scriptstyle \pm 0.002}$ & $0.872 {\scriptstyle \pm 0.001}$ & $0.872 {\scriptstyle \pm 0.001}$ & $0.872 {\scriptstyle \pm 0.001}$ &  & Oracle & $0.665 {\scriptstyle \pm 0.006}$ & $0.672 {\scriptstyle \pm 0.010}$ & $0.680 {\scriptstyle \pm 0.012}$ & $0.681 {\scriptstyle \pm 0.010}$ \\ \arrayrulecolor{lightgray} \hline \arrayrulecolor{black}
Var & $0.679 {\scriptstyle \pm 0.004}$ & $0.823 {\scriptstyle \pm 0.004}$ & $0.873 {\scriptstyle \pm 0.001}$ & $0.874 {\scriptstyle \pm 0.001}$ &  & ESM-Max & $0.564 {\scriptstyle \pm 0.005}$ & $0.649 {\scriptstyle \pm 0.004}$ & $0.655 {\scriptstyle \pm 0.007}$ & $0.676 {\scriptstyle \pm 0.008}$ \\
$\textsc{SPEC}$-$\varphi_3$ & $0.476 {\scriptstyle \pm 0.009}$ & $0.421 {\scriptstyle \pm 0.012}$ & $0.696 {\scriptstyle \pm 0.001}$ & $0.819 {\scriptstyle \pm 0.004}$ &  & ESM-Min & $0.582 {\scriptstyle \pm 0.005}$ & $0.569 {\scriptstyle \pm 0.008}$ & $0.597 {\scriptstyle \pm 0.015}$ & \best{$0.684 {\scriptstyle \pm 0.009}$} \\
MCFS & $0.694 {\scriptstyle \pm 0.000}$ & $0.762 {\scriptstyle \pm 0.004}$ & $0.840 {\scriptstyle \pm 0.006}$ & $0.813 {\scriptstyle \pm 0.009}$ &  & \texttt{GIR} & $0.607 {\scriptstyle \pm 0.004}$ & $0.636 {\scriptstyle \pm 0.019}$ & \best{$0.678 {\scriptstyle \pm 0.013}$} & \secbest{$0.683 {\scriptstyle \pm 0.009}$} \\
Random & $0.500 {\scriptstyle \pm 0.007}$ & $0.516 {\scriptstyle \pm 0.010}$ & $0.566 {\scriptstyle \pm 0.005}$ & $0.608 {\scriptstyle \pm 0.013}$ &  & Random & $0.573 {\scriptstyle \pm 0.012}$ & $0.611 {\scriptstyle \pm 0.008}$ & $0.640 {\scriptstyle \pm 0.011}$ & $0.661 {\scriptstyle \pm 0.005}$ \\ \arrayrulecolor{lightgray} \hline \arrayrulecolor{black}
\ours & \secbest{$0.808 {\scriptstyle \pm 0.003}$} & \best{$0.880 {\scriptstyle \pm 0.001}$} & \best{$0.884 {\scriptstyle \pm 0.002}$} & \best{$0.884 {\scriptstyle \pm 0.002}$} &  & \ours & \best{$0.658 {\scriptstyle \pm 0.006}$} & \best{$0.671 {\scriptstyle \pm 0.009}$} & \secbest{$0.674 {\scriptstyle \pm 0.009}$} & \secbest{$0.683 {\scriptstyle \pm 0.016}$} \\
\ours ($m=$10\%) & \best{$0.812 {\scriptstyle \pm 0.003}$} & \best{$0.880 {\scriptstyle \pm 0.002}$} & \best{$0.884 {\scriptstyle \pm 0.003}$} & \best{$0.884 {\scriptstyle \pm 0.002}$} &  & \ours ($m=$10\%) & \secbest{$0.655 {\scriptstyle \pm 0.015}$} & \best{$0.671 {\scriptstyle \pm 0.014}$} & $0.669 {\scriptstyle \pm 0.013}$ & $0.678 {\scriptstyle \pm 0.012}$ \\
\ours ($m=$20\%) & \secbest{$0.808 {\scriptstyle \pm 0.003}$} & $\secbest{0.871 {\scriptstyle \pm 0.002}}$ & \secbest{$0.883 {\scriptstyle \pm 0.002}$} & \secbest{$0.883 {\scriptstyle \pm 0.002}$} &  & \ours ($m=$20\%) & $0.641 {\scriptstyle \pm 0.009}$ & \secbest{$0.658 {\scriptstyle \pm 0.009}$} & $0.660 {\scriptstyle \pm 0.010}$ & $0.672 {\scriptstyle \pm 0.010}$ \\
\ours ($m=$30\%) & $0.783 {\scriptstyle \pm 0.004}$ & $0.865 {\scriptstyle \pm 0.004}$ & $0.882 {\scriptstyle \pm 0.002}$ & $0.882 {\scriptstyle \pm 0.002}$ &  & \ours ($m=$30\%) & $0.615 {\scriptstyle \pm 0.007}$ & $0.653 {\scriptstyle \pm 0.011}$ & $0.656 {\scriptstyle \pm 0.008}$ & $0.671 {\scriptstyle \pm 0.018}$ \\
\ours ($m=$40\%) & $0.751 {\scriptstyle \pm 0.005}$ & $0.851 {\scriptstyle \pm 0.006}$ & $0.864 {\scriptstyle \pm 0.004}$ & $0.871 {\scriptstyle \pm 0.003}$ &  & \ours ($m=$40\%) & $0.609 {\scriptstyle \pm 0.010}$ & $0.636 {\scriptstyle \pm 0.011}$ & $0.640 {\scriptstyle \pm 0.007}$ & $0.664 {\scriptstyle \pm 0.009}$ \\
\ours ($m=$50\%) & $0.733 {\scriptstyle \pm 0.007}$ & $0.832 {\scriptstyle \pm 0.005}$ & $0.861 {\scriptstyle \pm 0.004}$ & $0.872 {\scriptstyle \pm 0.002}$ &  & \ours ($m=$50\%) & $0.595 {\scriptstyle \pm 0.012}$ & $0.622 {\scriptstyle \pm 0.008}$ & $0.636 {\scriptstyle \pm 0.007}$ & $0.669 {\scriptstyle \pm 0.015}$ \\
\bottomrule
\end{tabular}
}
\label{tab:incomplete_data}

\end{table}

\subsection{Comparison with Adaptive Acquisition}
\label{subsec:adaptive}
\begin{wraptable}{r}{0.4\textwidth}
\vspace{-5mm}
\centering
\setlength{\aboverulesep}{0pt}
\setlength{\belowrulesep}{0pt}
\caption{\small \textbf{Comparison with unsupervised adaptive acquisition on \minib.} Accuracy (\%) at different acquisition budgets. \legendbox{blue!40} denotes best and \legendbox{pink!40} second best performance.}
\label{tab:adaptive}
\adjustbox{max width=\linewidth}{
\begin{tabular}{lccc}
\toprule
\textbf{Method} & \textbf{20\%} & \textbf{40\%} & \textbf{60\%} \\ \midrule
CIG & \secbest{$79.80$} & \secbest{$84.87$} & $86.17$ \\
ALG & $77.50$ & $83.60$ & $86.12$ \\
GSMRL-AIR & $79.18$ & $84.78$ & \secbest{$86.40$} \\ \arrayrulecolor{lightgray} \hline \arrayrulecolor{black}
\ours & \best{$81.70$} & \best{$86.87$} & \best{$87.60$} \\
\bottomrule
\end{tabular}
}
\vspace{-4mm}
\end{wraptable}
Beyond fixed acquisition orderings, we compare \ours with label-free, instance-adaptive policies under the same information setting: GSMRL-AIR~\citep{li2021active}, based on conditional generative imputations; Adaptive Latent Gain (ALG), which predicts one-step representation gain; and Cluster Information Gain (CIG), which estimates view-wise information gain under a Gaussian mixture model. \ours outperforms all three across every budget on \minib (Table~\ref{tab:adaptive}) and \hp(\S\ref{appendix:adaptive_baselines}). Notably, ALG shares \ours’s state and reward but optimizes only one-step gain, highlighting the benefit of long-horizon acquisition. Finally, \ours remains competitive with supervised acquisition policies that have access to downstream labels (\S~\ref{appendix:supervised}), quantifying the cost of task agnosticism.

\subsection{Comparison of representation-recovery}
\vspace{-5pt}
\label{subsec:recovery}
\begin{wraptable}{r}{0.42\textwidth}
\vspace{-5mm}
\centering
\setlength{\aboverulesep}{0pt}
\setlength{\belowrulesep}{0pt}
\caption{\small \textbf{Results of representation recovery.} Comparison of $\rho_t$ of \ours and \texttt{Variance} at different budgets. Higher $\rho_t$ denotes better recovery. \legendbox{blue!40} denotes best performance.}
\label{tab:recovery}
\adjustbox{max width=\linewidth}{
\begin{tabular}{llcccc}
\toprule
\textbf{Dataset} & \textbf{Method} & \textbf{20\%} & \textbf{40\%} & \textbf{60\%} & \textbf{80\%} \\ \midrule
\multirow{2}{*}{\hp} & Var & $0.38$ & $0.53$ & $0.65$ & $0.84$ \\
 & \ours & \best{$0.41$} & \best{$0.58$} & \best{$0.70$} & \best{$0.89$} \\ \arrayrulecolor{lightgray} \hline \arrayrulecolor{black}
\multirow{2}{*}{\minib} & Var & $0.35$ & $0.54$ & $0.60$ & $0.73$ \\
 & \ours & \best{$0.41$} & \best{$0.60$} & \best{$0.71$} & \best{$0.80$} \\
\bottomrule
\end{tabular}
}
\end{wraptable}

While our main results focus on downstream performance, we also directly evaluate the representation-recovery objective optimized by \ours (Eq.~\ref{eq:echok_reward}). For each instance $i$, we define normalized recovery after $t$ acquisitions as
\begin{equation*}
\rho_{i,t} := 1 - \frac{\|\zb_i(\Pb_t^i) - \zb_i^{\star}\|_2}{\|\zb_i(\varnothing) - \zb_i^{\star}\|_2},
\end{equation*}
We report $\rho_t$ averaged over test instances. Table~\ref{tab:recovery} shows that \ours achieves higher recovery than \texttt{Variance}, the strongest competing baseline on both \hp and \minib, across all budgets. Moreover, recovery and downstream accuracy exhibit the same ordering (Fig.~\ref{fig:scrnaseq}b; Fig.~\ref{fig:bigfig}): better preservation of the full-view representation corresponds to better downstream performance. This provides empirical support for the connection established in Theorem~\ref{theorem:loss_bound} between representation recovery and downstream error.

\vspace{-5pt}
\section{Conclusion}
\vspace{-5pt}
We presented \textsc{echo}-$k$, a task-agnostic and label-free framework for sequential test-time view acquisition that uses a pretrained model’s latent representation as a self-supervised proxy: under a cost budget, acquire views so that the induced partial representation rapidly approaches the full-view representation. We support this formulation with a theoretical analysis in a linear setting that characterizes the gap between oracle and feasible test-time policies, and we introduce a practical cost-aware RL algorithm that learns long-horizon acquisition strategies. Experiments with controlled synthetic data and multimodal foundation models across single-cell transcriptomics, spatial proteomics, and geospatial sensing show improved budgeted acquisition, provide evidence for transfer to downstream tasks unseen during policy learning, robustness to partially observed reference data, and extensibility to heterogeneous acquisition costs.
The broader implication is that foundation models need not remain passive predictors over pre-collected inputs; they can become resource-aware systems that actively shape the evidence available for inference.

\section*{Acknowledgments}
This work was supported by the Swiss AI Initiative through a Swiss National Supercomputing Centre (CSCS) allocation under the Swiss AI Large Grants \texttt{a146} and \texttt{ab007} (“Virtual Patient Platform”), an EPFL AI Center Postdoctoral Fellowship, an AI2050 Early Career Fellowship, and by the ISREC Foundation.

\bibliographystyle{icml2026}
\bibliography{ref_final}
\appendix 
\appendix
\onecolumn
\newpage
\section*{\Large Appendix}
\etocdepthtag.toc{mtappendix}
\etocsettagdepth{mtmain}{none}
\etocsettagdepth{mtappendix}{subsection}
\etocsettocstyle{}{}
\tableofcontents
\newpage
\newcommand{\filt}{\mathcal{F}}
\newcommand{\proxyset}{P^\star}
\newcommand{\prob}{\mathbb{P}}
\newcommand{\mathbbm}[1]{\text{\usefont{U}{bbm}{m}{n}#1}}

\section{Impact and Limitations}
\textsc{echo}-$k$ enables multimodal representation models to operate under test-time measurement budgets by learning task-agnostic, label-free acquisition policies. Instead of requiring downstream labels to decide which view to acquire next, the method uses the model’s own full-view representation as a self-supervised reference. This addresses a practical deployment bottleneck: multimodal and scientific foundation models often benefit from additional inputs, but acquiring all modalities can be infeasible due to cost, latency, limited tissue or sample material, scarce instrumentation time, or experimental constraints. In scientific settings, \textsc{echo}-$k$ could help prioritize which measurements, assays, markers, sensors, or experiments to run next under limited resources. In clinical settings, it could support budget-aware acquisition of follow-up tests or multimodal patient measurements, while keeping downstream decision-making within validated and expert-supervised workflows.
The main limitations concern deployment specificity rather than the general formulation. Acquisition costs and constraints may vary across domains, so practical use requires defining the relevant budget, availability, and cost model for the target setting. The method also benefits from reference representations that are informative enough to guide policy learning; our robustness experiments suggest that these references need not always be complete, but the appropriate level of missingness may depend on the application and backbone. Finally, as with any method intended for scientific or clinical use, policies should be evaluated on the intended endpoint and deployment distribution before adoption. Moreover, \ours preserves whatever information the pretrained backbone encodes, and does not assert that the full-view representation $\zb^{\star}$ is optimal for every downstream task; any task-irrelevant information in $\zb^{\star}$ is therefore inherited from the backbone. Our experiments evaluate transfer across downstream objectives, but do not test explicit distribution shift between the reference-construction data and the deployment domain. Similarly, the full-view representation is not always the best training target: masked references outperform it on MNIST-14 at 20\% and 40\% budgets and on \hp at 20\% (\S~\ref{appendix:additional_experiments}), suggesting that masking can act as a mild regularizer during training. Our robustness sweep considers randomly masked reference views and does not identify where partial references fail, since even at 50\% masking \ours remains competitive with strong baselines. A natural extension is structured incompleteness, where an entire assay or modality is missing for a subset of samples, and understanding how the reference representation interacts with that structure is an interesting direction for future work. These considerations leave substantial room for future work on richer cost models, uncertainty-aware acquisition, and prospective validation in real-world acquisition workflows.

\section{Related Work}
\label{app:related_work}

\textbf{Classical feature selection methods} aim to find a subset of relevant features that maximizes model performance while excluding irrelevancies~\citep{kohavi1997wrappers, guyon2003introduction}. Early surveys~\citep{blum1997selection} reviewed both empirical advances and theoretical definitions of feature relevance, formalizing the notions of strong and weak relevance. Several works cast the problem as submodular optimization~\citep{killamsetty2021grad, jain2023efficient, wei2014submodular,coleman2019selection}, though their end goal is to improve performance under limited budget during training rather than at test time. Previous work has also centered around greedy heuristics with theoretical guarantees for sequential search~\citep{golovin2011adaptive}. Another line of work focuses on instance-specific or personalized feature selection~\citep{melville2004active, li2018unsupervised}. Beyond selecting a subset of features, a related field is \textbf{feature importance quantification}: given a fixed dataset and (most commonly) a target, the goal is to isolate features that contribute significantly to predictions \citep{guyon2003introduction,settles2009active,li2017feature,lundberg2017unified,sristi2023contextual,jain2023efficient,lifshitz2025unsupervised}. Model-specific methods such as~\citep{breiman2001random} leverage the structure of the model to assign importance ranks, while model-agnostic methods~\citep{fisher2019all} rely on importance metrics applicable to any model.

\textbf{Active learning} (AL) incorporates a sample selection policy in the training loop in an online fashion: at each time $t$, a new sample is queried to augment the existing dataset \citep{settles2009active,zhan2022comparative,tharwat2023survey,li2024survey}. AL methods are particularly well suited for settings where acquiring new samples is expensive. They can be broadly split into three types of strategies \citep{zhan2022comparative}. \textit{Uncertainty-based} querying strategies select points based on the estimated information they provide to the learning algorithm, including maximum entropy selection \citep{shannon1948mathematical}, confidence-based selection \citep{wang2014new}, and adversarial methods \citep{ducoffe2018adversarial}. \textit{Diversity-based strategies} aim to increase the coverage of samples in the dataset, for instance through core-set approaches that build subsets approximating the full dataset \citep{sener2017active,claici2018wasserstein,feldman2019core,kim2022defense}. Lastly, \textit{combined} strategies mix both approaches \citep{yin2017deep,zhdanov2019diverse,shui2020deep}.

\textbf{Feature acquisition} (FA). Our method focuses on designing instance-specific \textit{feature} acquisition strategies rather than \textit{sample} acquisition policies \citep{settles2009active,kirsch2019batchbald,holzmuller2023framework,kolossov2023towards,gupte2024revisiting}. FA strategies can be broadly split \citep{rahbar2025survey} into \textit{greedy} strategies that sequentially acquire information by maximizing expected utility, and \textit{non-greedy} methods that plan acquisition over multiple time steps \citep{zubek2004pruning,dulac2011datum,shim2018joint,huang2018active,huang2025information,norcliffe2025stochastic,piskorz2025active}, often using reinforcement learning \citep{shim2018joint,zannone2019odin,kachuee2019opportunistic}. This problem has also been studied under the name of \textit{Bayesian experiment design} (BED), which aims to optimize the distribution of $\xi$, the random variable controlling the experiment, in order to maximize the information gain obtained by running the experiment parametrized by $\xi$ \citep{chaloner1995bayesian,sebastiani2000maximum,cavagnaro2010adaptive,ryan2016review,rainforth2024modern}. Classical BED methods sequentially update a model from a stream of incoming experiments; recently, the focus has shifted towards policy learning \citep{foster2021deep,blau2022optimizing,blau2023cross} and decision-aware experiment design \citep{huang2024amortized}.

\looseness -1 \textbf{Task-agnostic acquisition.} Contrarily to previous work, our proposed method is fundamentally task-agnostic. While some active learning and feature selection strategies have considered weakly or unlabeled settings \citep{ash2019deep,yehuda2022active, lifshitz2025unsupervised, kobayashi2025learning}, task-agnostic feature acquisition has received little attention. Closest to our setup are \citet{li2021active}, who consider a partially supervised setup and tackle this problem by training a surrogate reward model, and introduce an active instance recognition task in which the goal is to select features to reduce uncertainty.

\textbf{Theoretical aspects of FA} have received little attention, both in the supervised and unsupervised setting. \citet{chen2015sequential} study a probabilistic model inspired by Bayesian experiment design, in which the goal is to learn a policy that sequentially acquires features to infer the value of a target. \citet{kobayashi2025learning} and \citet{covert2023learning} provide theoretical insights in the supervised case. Our work is the first to consider the unsupervised case and to provide a bound on the gap between the greedy oracle and the optimal test-time policy.


\section{Extended Mathematical Framework}
\label{appendix:extended_math}

\paragraph{Data generation.} We consider $n$ i.i.d.\ instances. For each $i \in \intset{1}{n}$, 
the data $\xb^i = (\mb_1^i, \dots, \mb_K^i)$ is drawn from a joint distribution over 
$\Mcal^1 \times \cdots \times \Mcal^K$. We use \emph{view}, \emph{modality}, and 
\emph{feature group} interchangeably, as our framework applies uniformly to all settings.

\paragraph{Representation model.} Each view $\mb_k$ is processed by a view-specific encoder 
$\Phi_k: \Mcal^k \to \RR^{d_k}$. The encoded views are combined through a fixed aggregation 
function $\mathbf{A}_\mathbf{W}: \RR^{d_1} \times \dots \times \RR^{d_K} \to \RR^p$, 
yielding the full-view representation $\zb^* = \mathbf{A}_{\mathbf{W}}\!\left(\Phi_1(\mb_1), 
\dots, \Phi_K(\mb_K)\right)$. This subsumes both multimodal models (each $\mb_k$ is a 
distinct modality: image, audio, protein marker, \dots) and unimodal models (features 
partitioned into jointly acquired groups processed by a shared encoder). For notational 
simplicity, we drop the encoders in subsequent expressions and treat views as directly 
mapped into the latent space.

\paragraph{Sequential acquisition.} Let $\Pb_t \subseteq \intset{1}{K}$ denote the 
acquired set after $t$ steps and $\zb(\Pb_t) := \mathbf{A}_{\mathbf{W}}\!\left(\{\Phi_k(\mb_k): 
k \in \Pb_t\}\right)$ the induced partial representation, with unobserved views treated 
as missing. The finite-horizon acquisition problem is to select a sequence 
$(k_1,\ldots,k_T)$, with $k_{t+1}\notin \Pb_t$, minimizing the representation-recovery loss
\[
    \Lcal(\Pb_t) := \|\zb(\Pb_t)-\zb^\star\|_2^2.
\]
Equivalently, the policy maximizes cumulative progress toward $\zb^\star$ along the 
acquisition trajectory. The ideal one-step greedy rule is
\begin{equation} \label{eq:feat_aq_problem_app}
    k_{t+1} \in \argmin_{k \notin \Pb_t}
    \left\|\zb\!\left(\Pb_t \cup \{k\}\right) - \zb^*\right\|_2^2,
\end{equation}
which is not deployable at test time, since it requires access to both $\zb^\star$ and 
the representation after acquiring each candidate view.

\paragraph{Markov decision process formulation.} We frame sequential acquisition as an 
instance-specific MDP with a continuous state space and evolving action set. Dropping 
the instance superscript for clarity, the components are:

\begin{itemize}
    \item \textbf{State.} At step $t$, the state is $s_t = (M_t, \xb_{\Pb_t}, \zb(\Pb_t), b_t)$, 
    where $M_t \in \{0,1\}^K$ is the acquisition mask, $\xb_{\Pb_t}$ are the observed view 
    values, $\zb(\Pb_t)$ is the current partial representation, and $b_t$ is the remaining 
    budget. To maintain a fixed state space across steps, unobserved views are set to 
    $\texttt{NA}$, so the augmented space is 
    $\big(\Mcal^1 \cup \{\texttt{NA}\}\big) \times \cdots \times \big(\Mcal^K \cup \{\texttt{NA}\}\big)$ 
    \citep{norcliffe2025stochastic}.

    \item \textbf{Actions.} The available actions at step $t$ are the unacquired views: 
    $\Acal(s_t) = \{k \in \intset{1}{K} : k \notin \Pb_t\}$.

    \item \textbf{Transitions.} Upon selecting $k_{t+1}$, its value is drawn from 
    $\mb_{k_{t+1}} \sim p_{\texttt{data}}(\cdot \mid \mb_k, k \in \Pb_t)$, the state 
    updates to $\Pb_{t+1} = \Pb_t \cup \{k_{t+1}\}$, and $\zb(\Pb_{t+1})$ is recomputed. 
    The transition probability between states differing only in coordinate $k_t$ is
    \[
        \Prob\!\left[\Pb_t \mid \Pb_{t-1}\right] = 
        \Prob\!\left[a_t = k_t \mid \Pb_{t-1}\right] \cdot 
        p_{\texttt{data}}\!\left(\mb_{k_t} \mid \mb_k,\, k \in \Pb_{t-1}\right).
    \]

    \item \textbf{Reward.} The reward at step $t+1$ quantifies progress toward the 
    reference-view representation:
    \[
        r_{t+1}(\Pb_{t+1}) := \|\zb(\Pb_t) - \zb^\mathrm{ref}\|_2^2 - \|\zb(\Pb_{t+1}) - \zb^\mathrm{ref}\|_2^2.
    \]
    where $\zb^{\mathrm{ref}}$ can be the full view embedding $\zb^*$ or the embedding from all the views (can be a subset of full views) measured for the underlying data point.
\end{itemize}

\paragraph{Selection policies.} A stochastic selection policy is a map 
$\pi_\theta(\cdot \mid s_t) \in \Delta(\Acal(s_t))$, where $\Delta(\Acal(s_t))$ is the 
probability simplex over available actions. During training, fully observed samples 
allow $\zb^\star$ to be computed and used as a self-supervised target — no downstream 
labels are required. When only a reference subset $R_i$ is observed for instance $i$, 
$\zb(R_i)$ serves as a partial target. At deployment, the learned policy is fixed and 
acts solely from the current state $s_t$.

\paragraph{Connection to sequential decision making.} This MDP is an instance of a 
$K$-armed bandit with correlated arms and a single-pull constraint per arm 
\citep{gupta2021multi}; see also \citet{guha2007multi,chakrabarti2008mortal,carrara2019budgeted} 
for related setups. The sequential decision problem is Markovian with respect to 
$(\Pb_t, \zb(\Pb_t))$: given the current partial representation and acquired set, 
future acquisition values are independent of the full history of selections.

\section{Notation}
\label{appendix:notation}

\renewcommand{\arraystretch}{1.12}
\begin{longtable}{>{\raggedright\arraybackslash}p{0.24\textwidth}
                  >{\raggedright\arraybackslash}p{0.68\textwidth}}
\toprule
\textbf{Notation} & \textbf{Meaning} \\
\midrule
\endfirsthead

\toprule
\textbf{Notation} & \textbf{Meaning} \\
\midrule
\endhead

\bottomrule
\endfoot

\multicolumn{2}{l}{\textbf{Representation-acquisition setup}} \\
\midrule

$n$ & Number of i.i.d. samples. \\

$K$ & Number of possible views in the general acquisition problem. \\

$\mb_k^i$ & The $k$-th view of sample $i$. A view may be a modality or a jointly acquired feature group. \\

$\Mcal^k$ & Space of possible values for view $k$. \\

$\Phi_k : \Mcal^k \to \RR^{d_k}$ & View-specific encoder for view $k$. \\

$\mathbf{A}_{\mathbf{W}}$ & Fixed aggregation map combining encoded views into a representation. \\

$\zb^\star$ & Full-view representation, obtained from all available views. \\

$\Pb_t \subseteq \{1,\dots,K\}$ & Set of acquired views after $t$ acquisition steps in the general MDP formulation. \\

$\zb(\Pb_t)$ & Partial representation induced by the acquired set $\Pb_t$. \\

$T$ & Acquisition budget. \\

$s_t = (\Pb_t,\zb(\Pb_t),b_t)$ & Deployment state of the acquisition policy. \\

$b_t$ & Remaining or current budget component in the MDP state. \\

$\Lcal(\Pb_t)$ & Representation-recovery loss, $\Lcal(\Pb_t)=\|\zb(\Pb_t)-\zb^\star\|_2^2$. \\

\midrule
\multicolumn{2}{l}{\textbf{Stylized linear model}} \\
\midrule

$d$ & Number of views/features in the stylized linear theory. \\

$m$ & Dimension of the latent linear representation. \\

$\xb \in \RR^d$ & Random view vector in the linear theory. \\

$x_j$ & The $j$-th component of the view vector $\xb$. \\

$\Wm \in \RR^{m \times d}$ & Fixed linear aggregation matrix. \\

$\wb_j$ & The $j$-th column of $\Wm$. \\

$\zb^* = \Wm\xb$ & Latent representation in the linear theory. \\

$\Ical_t$ & Information available to a policy at time $t$. This differs between oracle and test-time policies. \\

$\Fcal_t = \sigma(x_k : k \in \Pb_t)$ & Filtration generated by the views observed up to time $t$. \\

$\mathcal{L}(\Pb_t \mid \Ical_t)$ & Conditional expected squared reconstruction error of selected set $\Pb_t$ given information $\Ical_t$. \\

\midrule
\multicolumn{2}{l}{\textbf{Policies and selected sets}} \\
\midrule

$\Pb_t^\star$ & Set selected by the oracle policy after $t$ steps. \\

$\widehat{\Pb}_t$  & Set selected by the optimal test-time policy after $t$ steps. \\

$k_t^\star$ & View selected by the oracle policy at step $t$. \\

$\widehat{k}_t$ & View selected by the optimal test-time policy at step $t$. \\

$\ttset_{-s}$ & Set of views still available to the test-time policy at step $s$. \\

\midrule
\multicolumn{2}{l}{\textbf{Losses, energies, and regrets}} \\
\midrule

$\Eb_j$ & Energy of view $j$, defined by $\Eb_j := \|\wb_j\|_2^2 x_j^2$. \\

$\Cb_{jj'}$ & Cross-view interaction term, defined by $\Cb_{jj'} := \wb_j^\top \wb_{j'}x_jx_{j'}$. \\

$\trueloss(\Pb_t,\xb)$ & True latent reconstruction loss of selected set $\Pb_t$ on realization $\xb$. \\

$\proxyloss(\Pb_t,\xb)$ & Proxy loss based only on marginal energy terms. \\

$\mathcal{R}_t$ & True latent regret of the optimal test-time policy relative to the oracle policy. \\

$\widetilde{\mathcal{R}}_t$ & Proxy regret of the optimal test-time policy relative to the oracle policy. \\

$\Yb_{j,s}$ & Centered energy error at step $s$, $\Yb_{j,s}:=\Eb_j-\EE[\Eb_j\mid\Fcal_s]$. \\

$\Rb_s$ & Maximal conditional deviation spread among available views at step $s$. \\

$M_j^{(p)}$ & $p$-th moment of the energy of view $j$, $M_j^{(p)}:=\EE[\Eb_j^p]$. \\

$M_{(1)}^{(p)} \ge \cdots \ge M_{(d)}^{(p)}$ & Sorted $p$-th energy moments. \\

$A_t$ & Moment envelope,
$A_t := \sum_{i=1}^d M_{(i)}^{(p)} \min(t,d-i+1)$. \\

$p,q$ & H\"older conjugate exponents, with $p \in (1,1/(2\sigma))$ and $q=p/(p-1)$. \\

\midrule
\multicolumn{2}{l}{\textbf{Mismatch and ranking events}} \\
\midrule

$G_t$ & Views selected only by the oracle policy at time $t$. \\

$B_t$ & Views selected only by the optimal test-time policy at time $t$. \\

$\rank_A(k)$ & Rank of $\Eb_k$ among the energy values $(\Eb_j)_{j\in A}$. \\

$\bad(s,t)$ & Bad ranking event,
$\bad(s,t):=\{\rank_{\ttset_{-s}}(\widehat{k}_{s+1})>t-s\}$. \\

$D_t(\lambda)$ & Disagreement envelope controlling the expected number of oracle--test-time mismatches. \\

$N_t$ & Effective size term in the true regret bound, $N_t:=d-t+D_t$. \\

\midrule
\multicolumn{2}{l}{\textbf{Gumbel latent correlation model}} \\
\midrule

$Z$ & Discrete latent factor governing correlations between views. \\

$z_1,\dots,z_M$ & Possible latent states of $Z$. \\

$p_m$ & Prior probability of latent state $z_m$. \\

$p_{\min}$ & Minimum latent-state probability, $p_{\min}:=\min_{m\in[M]}p_m$. \\

$\bm{\eta},\bm{\beta}\in\RR^d$ & Deterministic location and latent-loading vectors in the Gumbel model. \\

$\varepsilon_j$ & Independent Gumbel noise term for view $j$. \\

$\sigma$ & Gumbel scale parameter. \\

$a_j$ & Squared column norm, $a_j:=\|\wb_j\|_2^2$, when writing utilities in the form $\log a_j+2\eta_j+2\beta_jz_m$. \\

$u_j^{(m)}$ & Utility of view $j$ under latent state $z_m$,
$u_j^{(m)}:=\log a_j+2\eta_j+2\beta_jz_m$. \\

$\theta_j^{(m)}$ & Plackett--Luce weight,
$\theta_j^{(m)}:=\exp(u_j^{(m)}/(2\sigma))$. \\

$\gamma$ & Minimum utility separation,
$\gamma:=\min_{m\in[M]}\min_{j\ne j'}|u_j^{(m)}-u_{j'}^{(m)}|$. \\

$R_{\max}$ & Maximum latent-state variation in view weights,
$R_{\max}:=\max_{j\in[d]}\frac{\max_m e^{2\beta_jz_m}}{\min_m e^{2\beta_jz_m}}$. \\

$\varepsilon_\gamma$ & Posterior accuracy threshold,
$\varepsilon_\gamma:=\frac{e^\gamma-1}{R_{\max}+e^\gamma-1}$. \\

$\rho_{\mathrm{PL}}(k)$ & Minimum Plackett--Luce selection probability over subsets of size $k$. \\

$\rho_\star$ & Ranking-concentration parameter, $\rho_\star:=1-e^{-\gamma/(2\sigma)}$. \\

$\alpha_m(\Fcal_s)$ & Posterior probability of latent state $z_m$ after observing $\Fcal_s$,
$\alpha_m(\Fcal_s):=\prob[Z=z_m\mid\Fcal_s]$. \\

$E_{\mathrm{post}}(s,z_m)$ & Posterior failure event under true state $z_m$. \\

$c_j(m,m';\lambda)$ & Chernoff coefficient between the conditional distributions of view $j$ under states $m$ and $m'$. \\

$\rho_\lambda$ & Worst-case Chernoff contraction coefficient,
$\rho_\lambda:=\max_{j\in[d]}\max_{m\ne m'}c_j(m,m';\lambda)$. \\

$C_\lambda$ & Posterior concentration constant,
$C_\lambda:=(M-1)\left(\frac{1-\varepsilon_\gamma}{\varepsilon_\gamma p_{\min}}\right)^\lambda$. \\

\midrule
\multicolumn{2}{l}{\textbf{Spectral misalignment quantities}} \\
\midrule

$\Mcal(\Pb)$ & Difference between true and proxy losses for selected set $\Pb$,
$\Mcal(\Pb):=\trueloss(\Pb)-\proxyloss(\Pb)$. \\

$u(\Pb)$ & Vector of unobserved weighted view magnitudes,
$u(\Pb):=[\mathbf{1}_{\{j\notin\Pb\}}\|\wb_j\|x_j]_{j=1}^d$. \\

$\bm{U}(\Pb,\Pb')$ & Set of indices not jointly selected by $\Pb$ and $\Pb'$,
$\bm{U}(\Pb,\Pb'):=(\Pb\cap\Pb')^c$. \\

$\Gamma(\Pb,\Pb')$ & Symmetric difference between two selected sets,
$\Gamma(\Pb,\Pb'):=\Pb\Delta\Pb'$. \\

$\Bb_J$ & Matrix of normalized embedding columns indexed by $J$,
$\Bb_J := (\wb_u/\|\wb_u\|)_{u\in J}$. \\

$p_{\mathrm{eff}}$ & Effective dimension parameter in the near-isometry assumption. \\

$C,C'$ & Constants appearing in the near-isometry and spectral misalignment bounds. \\

$B$ & Matrix of all normalized embedding columns, $B:=[\wb_j/\|\wb_j\|]_{j=1}^d$. \\

$H$ & Gram deviation matrix, $H:=B^\top B-I$. \\

$H_U$ & Restriction of the Gram deviation matrix to the index set $\bm{U}(\Pb,\Pb')$. \\

\end{longtable}

\section{Proofs}
\label{appendix:proofs}

\subsection{Proof of Lemma \ref{lemma:diff_proxy_loss}}

\begin{proof}
    Let $G_t$ be the set of features selected only by the oracle policy at time $t$, and $B_t$ the set of features selected only by the optimal test-time policy. First, observe that we have the trivial bound
    \begin{align*}
        | B_t | = | G_t | \leq \min(d-t,t)  
    \end{align*}
    We have
    \begin{align*}
        \proxyloss(\ttset_t,\xb) - \proxyloss(\oracleset_t,\xb) = \sum\limits_{g \in G_t} \Eb_g - \sum\limits_{b \in B_t}\Eb_b.
    \end{align*}
    Write $\widehat{k}_2,\dots,\widehat{k}_t$ for the indices of the features selected by the optimal test-time policy. Since $G_t$ and $B_t$ are of equal size, there exists a permutation $\sigma$ between $B_t$ and $G_t$. Also, by definition of $G_t$, we know that $G_t \cap B_t = \{\varnothing\}$. By iterating over time steps, we have 
    \begin{align*}
        \sum\limits_{g \in G_t} \Eb_g - \sum\limits_{b \in B_t}\Eb_b \leq \sum\limits_{s=2}^{t} \bm{1}_{\{\widehat{k}_s \in B_t\}}\Big(\Eb_{\sigma(\widehat{k}_s)} - \Eb_{\widehat{k}_s}\Big).
    \end{align*}
    Now, we also have 
    \begin{align*}
        &\sum\limits_{s=2}^{t} \bm{1}_{\{\widehat{k}_s \in B_t\}}\Big(\Eb_{\sigma(\widehat{k}_s)} - \Eb_{\widehat{k}_s}\Big)\\ 
        = &\sum\limits_{s=2}^{t} \bm{1}_{\{\widehat{k}_s \in B_t\}}\Big(\Eb_{\sigma(\widehat{k}_s)} - \EE \big[\Eb_{\sigma(\widehat{k}_s)} \mid \filt_s \big] + \EE \big[\Eb_{\sigma(\widehat{k}_s)} \mid \filt_s \big] - \Eb_{\widehat{k}_s}-\EE \big[\Eb_{\widehat{k}_s} \mid \filt_s \big] + \EE \big[\Eb_{\widehat{k}_s} \mid \filt_s \big]\Big)\\
        =& \sum\limits_{s=2}^{t} \bm{1}_{\{\widehat{k}_s \in B_t\}} \Big(Y_{\sigma(\widehat{k}_s),s} - Y_{\widehat{k}_s,s} + \underbrace{\EE \big[\Eb_{\sigma(\widehat{k}_s)} \mid \filt_s \big] -\EE \big[\Eb_{\widehat{k}_s} \mid \filt_s \big]}_{\leq 0} \Big)\\
        \leq & \sum\limits_{s=2}^{t} \bm{1}_{\{\widehat{k}_s \in B_t\}} \Big(Y_{\sigma(\widehat{k}_s),s} - Y_{\widehat{k}_s,s} \Big).
    \end{align*}
    where the last equality comes from optimality of the test time policy and the fact that $\sigma(\widehat{k}_s)$ is available, i.e. $\sigma(\widehat{k}_s) \in \ttset_{-s}$ . By definition of $\Rb_s$, we have
    \begin{align*}
        \proxyloss(\ttset_t,\xb) - \proxyloss(\oracleset_t,\xb) \leq \sum\limits_{s=2}^{t} \bm{1}_{\{\widehat{k}_s \in B_t\}} \Rb_s.
    \end{align*}
\end{proof}

\subsection{Proof of Proposition C.1}

We prove the following Proposition. 

\begin{proposition}
\label{prop:proxy_bound_app}
    Take $p \in \big(1,\frac{1}{2\sigma}\big)$ and let $q = \frac{p}{p-1}$. Under Assumptions, we have 
    \begin{align*}
      \widetilde{\mathcal{R}}_t \leq A_t ^{1/p} \left( \min \Big\{t,d-t, \frac{C_\lambda}{1-\rho_\lambda} + \frac{1}{e^{\rho_0}-1}\Big\}\right)^{1/q}
\end{align*}
where $A_t = \sum_{i=1}^d M_{(i)} \min(t, d-i+1)$. 
\end{proposition}

\begin{proof}
    By Lemma \ref{lemma:diff_proxy_loss}, we have 
    \begin{align*}
        \widetilde{\mathcal{R}}_t \leq \sum\limits_{s=2}^{t} \EE_{\xb}\big[\bm{1}_{\{\widehat{k}_s \in B_t\}} \Rb_s\big].
    \end{align*}
    By H\"older's inequality, we have 
    \begin{align*}
        \EE_{\xb}\big[\bm{1}_{\{\widehat{k}_s \in B_t\}} \Rb_s\big] \leq & \EE_{\xb}\big[ \Rb^p_s\big]^\frac{1}{p} \EE_{\xb}\big[ \bm{1}_{\widehat{k}_s \in B_t}^q\big]^\frac{1}{q}\\
        =&\EE_{\xb}\big[ \Rb^p_s\big]^\frac{1}{p} \EE_{\xb}\big[ \bm{1}_{\widehat{k}_s \in B_t}\big]^\frac{1}{q}
    \end{align*}
    for any $p \in \big(1,\frac{1}{2\sigma}\big)$ and $q=\frac{p}{p-1}$. Using Lemma \ref{technical_lemma_bad_event}, we obtain 
    \begin{align*}
        \EE_{\xb}\big[ \Rb^p_s\big]^\frac{1}{p} \EE_{\xb}\big[ \bm{1}_{\widehat{k}_s \in B_t}\big]^\frac{1}{q} \leq \EE_{\xb}\big[ \Rb^p_s\big]^\frac{1}{p}\EE_{\xb}\big[ \bm{1}_{\bad(s,t)}\big]^\frac{1}{q}
    \end{align*}

    Now, using Lemma \ref{tech_lem_1}, we have
    \begin{align*}
        \EE_{\xb}\big[ \Rb^p_s\big]^\frac{1}{p} \leq 4\left(\sum_{i=1}^{d-s}M^{(p)}_{(i)}\right)^{\frac{1}{p}} 
    \end{align*}
    where $M^{(p)}_{(j)} := \EE[\Eb_j^p]$, and $M^{(p)}_{(1)} \geq ... \geq M^{(p)}_{(d)}$ are the sorted $p$-moments. Hence, we have
    \begin{align*}
         \widetilde{\mathcal{R}}_t \leq \sum\limits_{s=2}^{t} \EE_{\xb}\big[ \Rb^p_s\big]^\frac{1}{p}\EE_{\xb}\big[ \bm{1}_{\bad(s,t)}\big]^\frac{1}{q} \leq \sum\limits_{s=2}^{t}4\left(\sum_{i=1}^{d-s}M^{(p)}_{(i)}\right)^{\frac{1}{p}}\EE_{\xb}\big[ \bm{1}_{\bad(s,t)}\big]^\frac{1}{q}. 
    \end{align*}
    For any $s \leq d$, let $S_{d-s} := \sum_{i=1}^{d-s} M^{(p)}_{(i)}$ and $A_t = \sum_{s=0}^{t-1} S_{d-s} = \sum_{s=0}^{t-1} \sum_{i=1}^{d-s} M^{(p)}_{(i)} = \sum_{i=1}^d M^{(p)}_{(i)} \min(t, d-i+1)$. Then, 
\[
\sum\limits_{s=2}^{t}4\left(\sum_{i=1}^{d-s}M^{(p)}_{(i)}\right)^{\frac{1}{p}}\EE_{\xb}\big[ \bm{1}_{\bad(s,t)}\big]^\frac{1}{q} \leq A_t ^{1/p} \left(\sum_{s=0}^{t-1} \EE_{\xb}\big[ \bm{1}_{\bad(s,t)}\big]\right)^{1/q} = A_t ^{1/p} \left(\sum_{s=0}^{t-1} \prob\big({\bad(s,t)}\big)\right)^{1/q}
\]
by H\"older's inequality. Now, combining this with Lemma \ref{technical_lemma_bound_bad}, we have the upper bound
\begin{align*}
    A_t ^{1/p} \left(\sum_{s=0}^{t-1} \prob\big({\bad(s,t)}\big)\right)^{1/q} \leq A_t ^{1/p} \left(\min \Bigg\{t,d-t,\sum_{s=0}^{t-1} \mathbb{E}_{Z}\Big[\prob\big(E_{\mathrm{post}}(s,z_{m^\star})\big)\Big] + e^{-(t-s) \rho_{\mathrm{PL}}(d-t)}\Bigg\}\right)^{1/q}
\end{align*}
which can in turn be bounded from above by 
\begin{align*}
    A_t ^{1/p} \left(\sum_{s=0}^{t-1} \mathbb{E}_{Z}\Big[\prob\big(E_{\mathrm{post}}(s,z_{m^\star})\big)\Big] + e^{-k \rho_{\mathrm{PL}}(d-t)}\right)^{1/q} \leq A_t ^{1/p} \left( \min \Bigg\{t,d-t, \sum_{s=1}^{t} C_\lambda \rho_\lambda^s + e^{-(t-s) \rho_{\mathrm{PL}}(d-t)}\Bigg\}\right)^{1/q}
\end{align*}
where $C_\lambda := (M - 1) \left( \frac{1 - \varepsilon_\gamma}{\varepsilon_\gamma p_{\min}} \right)^\lambda$ using Lemma \ref{technical_lemma:posterior-concentration}. The sum in the right-hand side term of this bound can be further upper bounded by 
\begin{align*}
    \frac{C_\lambda}{1-\rho_\lambda} + \frac{1}{e^{\rho_0}-1}.
\end{align*}
This yields the final bound 
\begin{align*}
      \widetilde{\mathcal{R}}_t \leq A_t ^{1/p} \left( \min \Big\{t,d-t, \frac{C_\lambda}{1-\rho_\lambda} + \frac{1}{e^{\rho_0}-1}\Big\}\right)^{1/q}. 
\end{align*}
\end{proof}

\subsection{Proof of Theorem \ref{theorem:loss_bound}}

\begin{proof}
    From Lemma \ref{tech_lemma_miss}, we have  
    \begin{align*}
        \trueloss(\ttset_t,\xb) - \trueloss(\oracleset_t,\xb) &\leq \proxyloss(\ttset_t,\xb) - \proxyloss(\oracleset_t,\xb) \\&+ C' \left(\sqrt{\frac{|\bm{U}(\bm{P},\bm{P}')|}{p_{\text{eff}}} } + \frac{|\bm{U}(\bm{P},\bm{P}')|}{{p}_{\text{eff}}}\right) \sqrt{\sum_{j \in \bm{U}(\bm{P},\bm{P}')} \Eb_j}\sqrt{\sum_{j \in \Gamma(\bm{P},\bm{P}')} \Eb_j}. 
    \end{align*}
    Taking expectations over $\xb$ and using the Cauchy-Schwarz inequality, we obtain
    \begin{align*}
        \mathcal{R}_t \leq \widetilde{\mathcal{R}}_t + C' \Bigg(\frac{\EE_{\xb}\Big(\sqrt{|\bm{U}(\bm{P},\bm{P}')|}\Big)}{p_{\text{eff}}} + \frac{\EE_{\xb}\Big(|\bm{U}(\bm{P},\bm{P}')|\Big)}{p_{\text{eff}}}\Bigg) \EE_{\xb} \sqrt{ \sum_{j \in \bm{U}(\bm{P},\bm{P}')} \Eb_j} \EE_{\xb} \sqrt{ \sum_{j \in \Gamma(\bm{P},\bm{P}')} \Eb_j}
    \end{align*}
    and using Jensen's inequality, we obtain 
    \begin{align*}
        \mathcal{R}_t \leq \widetilde{\mathcal{R}}_t + C' \Bigg(\frac{\sqrt{\EE_{\xb}|\bm{U}(\bm{P},\bm{P}')|}}{p_{\text{eff}}} + \frac{\EE_{\xb}\Big(|\bm{U}(\bm{P},\bm{P}')|\Big)}{p_{\text{eff}}}\Bigg)  \sqrt{ \EE_{\xb} \sum_{j \in \bm{U}(\bm{P},\bm{P}')} \Eb_j}. \sqrt{ \EE_{\xb} \sum_{j \in \Gamma(\bm{P},\bm{P}')} \Eb_j}.
    \end{align*}
    We can in turn bound this using \ref{technical_lemma_bound_bad} by 
    \begin{align*}
        & \widetilde{\mathcal{R}}_t + C' \Bigg(\frac{\sqrt{\EE_{\xb}|\bm{U}(\bm{P},\bm{P}')|}}{p_{\text{eff}}} + \frac{\EE_{\xb}\Big(|\bm{U}(\bm{P},\bm{P}')|\Big)}{p_{\text{eff}}}\Bigg)  \sqrt{ \EE_{\xb} \sum_{j \in \bm{U}(\bm{P},\bm{P}')} \Eb_j}\\ 
        \leq & \widetilde{\mathcal{R}}_t + C' \Bigg(\frac{\sqrt{d-t+D_t}}{p_{\text{eff}}} + \frac{d-t+D_t}{p_{\text{eff}}}\Bigg)  \sqrt{D_t(d-t+ D_t)}.
    \end{align*}
\end{proof}

\section{Technical Lemmas}
\label{appendix:technical_lemmas}

\subsection{Technical Lemma 1: Inclusion of Mismatch Events}

Recall that for any subset $A \subseteq \{1,\dots,d\}$, $\rank_{A}(k)$ is the rank of $\Eb_k$ among all $(\Eb_j)_{j \in A}$. Consider the event 
\begin{align*}
    \bad(s,t) := \big\{\rank_{\ttset_{-s}}(\widehat{k}_{s+1}) > t-s\big\} = \Big\{ \Eb_{\widehat{k}_{s+1}} < \Eb_{(t-s)}\Big\}. 
\end{align*}

\begin{lemma}
\label{technical_lemma_bad_event}
    We have
    \begin{align*}
        \sum\limits_{s=2}^{t} \bm{1}_{\{\widehat{k}_s \in B_t\}} \Rb_s \leq \sum\limits_{s=2}^{t} \bm{1}_{\bad(s,t)} \Rb_s. \vspace{-10pt}
    \end{align*}
\end{lemma}
\vspace{-10pt}
\begin{proof}
    At time $s \leq t$, the test-time policy has selected at most $s-1$ items from the top $t-1$ values of $(\Eb_k)_{k=1}^d$, the most favorable case being when it selects the same items as the oracle policy. Hence, there are at least $t-s-1$ items from the top $t-1$ values of $(\Eb_k)_{k=1}^d$ in the set $\ttset_{-s}$. Therefore, if $\widehat{k}_{s} \in B_t$, then by definition it is not in the top $t-s-1$ items, because it would otherwise also belong to $G_t = \oracleset_t\backslash B_t$. Hence,
    \begin{align*}
        \Eb_{\widehat{k}_s} \leq \Eb_{(t-s-1)}. 
    \end{align*}
    We have hence proved that $\{\widehat{k}_s \in B_t\} \subseteq \bad(s,t)$, which proves our claim. 
\end{proof}

\subsection{Technical Lemma 2: Moment Bound on Conditional Deviations}

\begin{lemma}
\label{tech_lem_1}
    Let $M^{(p)}_j := \EE[\Eb_j^p]$, and $M^{(p)}_{(1)} \geq ... \geq M^{(p)}_{(d)}$ be the sorted $(p)$-moment order. Then, for every $s$ and for any $p \in \big(1,\frac{1}{2\sigma}\big)$
\[
\EE[\Rb_s^p]^{1/p} \leq 4\left(\sum_{i=1}^{d-s}M_{(i)}\right)^{\frac{1}{p}}
\]
\end{lemma}

\begin{proof}
Recall that $\Rb_s :=  \max_{i,j \in \ttset_{-t}} (\Yb_{i,s} - \Yb_{j,s})$. We have
\[
\Rb_s \leq \max_{i,j \in \ttset_{-t}} \bigl( |\Yb_{i,s}| + |\Yb_{j,s}| \bigr)
\leq 2 \max_{i,j \in \ttset_{-t}} |\Yb_{j,s}|.
\]
Take $p \in \big(1,\frac{1}{2\sigma}\big)$. Since
$\max_{i,j \in \ttset_{-t}} |\Yb_{j,s}| \leq \Bigl( \sum_{j \in \ttset_{-t}} |\Yb_{j,s}|^p \Bigr)^{1/p}$,
we have $\Rb_s^p \le 2^p \sum_{j \in \ttset_{-t}} |\Yb_{j,s}|^p$. Using the inequality for any $k \in \{1,\dots,d\}$
\[
|\Eb_k - \mathbb{E}[\Eb_k \mid \filt_s]|^p
\le 2^{p-1} \bigl( |\Eb_k|^p + |\mathbb{E}[\Eb_k \mid \filt_s]|^p \bigr),
\]
together with Jensen's inequality, we obtain
$\mathbb{E}|\Yb_{j,s}|^p \leq 2^p \mathbb{E}[|\Eb_j|^p] = 2^p M^{(p)}_j$.
Therefore,
\[
\mathbb{E}[\Rb_s^p]
\le 2^p \sum_{j \in \ttset_{-t}} \mathbb{E}|\Yb_{j,s}|^p
\le 4^p \sum_{j \in \ttset_{-t}} M^{(p)}_j .
\]
Since $|\ttset_{-t}| = d - s$ deterministically and
$\sum_{j \in \ttset_{-t}} M^{(p)}_j \le \sum_{i=1}^{d-s} M^{(p)}_{(i)}$,
the claimed bound follows.
\end{proof}

\subsection{Technical Lemma 3: Bounding Disagreement Probability}

For each state $m \in \{1,\dots,M\}$, define
\begin{align*}
& u_j^{(m)} := \log a_j + 2\eta_j + 2\beta_j z_m, \qquad \theta_j^{(m)} := \exp\left( \frac{u_j^{(m)}}{2\sigma} \right).
\end{align*}
Also let 
$\gamma := \min_{m \in [M]} \min_{j \neq j'} \left| u_j^{(m)} - u_{j'}^{(m)} \right| > 0$
and 
$R_{\max} := \max_{j \in [d]} \frac{\max_m e^{2\beta_j z_m}}{\min_m e^{2\beta_j z_m}} < \infty$.
Finally, define $\varepsilon_\gamma := \frac{e^\gamma - 1}{R_{\max} + e^\gamma - 1} \in (0, 1)$
and 
$\rho_{\mathrm{PL}}(k) := \min_{m \in [M]} \min_{\substack{J \subset [d] \\ |J| = k}} \frac{\max_{j \in J} \theta_j^{(m)}}{\sum_{\ell \in J} \theta_\ell^{(m)}} \geq \frac{1}{k}$.
Lastly, define $\alpha_m(\filt_s) := \prob\Big[Z=z_m\mid \filt_s\Big]$ and 
\begin{align*}
    E_{\mathrm{post}}(s,z_m) := & \left\{ \prob\Big[Z=z_m\mid \filt_s\Big] < 1 - \varepsilon_\gamma \mid Z=z_m\right\}, \\
    E_{\mathrm{post}}(s,Z) := & \left\{ \prob\Big[Z=z_m\mid \filt_s\Big] < 1 - \varepsilon_\gamma \mid Z\right\}.
\end{align*}

\begin{lemma}
\label{technical_lemma_bound_bad}
    Assume that $\prob\big[Z=z_m\mid \filt_s\big] >0$ for all $z_m$ and $s\geq 1$. We have 
    \begin{align*}
        \prob(\bad(s,t)) &\leq \mathbb{E}_{Z}\Big[\prob\big(E_{\mathrm{post}}(s,Z)\big)\Big] + e^{-k \rho_{\mathrm{PL}}(d-t)}.
    \end{align*}
\end{lemma}

\begin{proof}
Let us first work conditionally on $Z = z_{m^\star}$. Define $j^\star(z_{m^\star}) := \arg\max_{j \in \ttset_{-t}}u_j^{(m^\star)}$. Since $\gamma > 0$, this index is uniquely defined. We claim that on $E_{\mathrm{post}}(s,z_{m^\star})^c$, the test-time optimal policy selects $\widehat{k}_{s+1} = j^\star(z_{m^\star})$. To see this, note that
$\mathbb{E}[\Eb_j \mid \mathcal{F}_s] \propto e^{\log a_j + 2\eta_j} \sum_{m=1}^M \alpha_m(\mathcal{F}_s) e^{2\beta_j z_m}$.
Define $T_j := \sum_{m=1}^M \alpha_m(\mathcal{F}_s) e^{2\beta_j z_m}$.

On $E_{\mathrm{post}}(s,z_{m^\star})^c$, we have $\alpha_{m^\star}(\mathcal{F}_s) \geq 1 - \varepsilon_\gamma$. Hence:
\begin{align*}
T_j = \alpha_{m^\star}(\mathcal{F}_s) e^{2\beta_j z_{m^\star}} + \underbrace{\sum_{m \neq m^\star} \alpha_m(\mathcal{F}_s) e^{2\beta_j z_m}}_{> 0} > (1 - \varepsilon_\gamma) e^{2\beta_j z_{m^\star}}
\end{align*}
where the strict inequality uses our assumption $\prob\big[Z=z_m\mid \filt_s\big] >0$. On the other hand,
\begin{align*}
T_j &\leq (1 - \varepsilon_\gamma) e^{2\beta_j z_{m^\star}} + \varepsilon_\gamma R_{\max} e^{2\beta_j z_{m^\star}} = e^{2\beta_j z_{m^\star}} \left( 1 - \varepsilon_\gamma + \varepsilon_\gamma R_{\max} \right).
\end{align*}
Thus, up to a positive multiplicative constant:
$(1 - \varepsilon_\gamma) e^{u_j^{(m^\star)}} \leq \mathbb{E}[\Eb_j \mid \mathcal{F}_s]  \leq \left( 1 - \varepsilon_\gamma + \varepsilon_\gamma R_{\max} \right) e^{u_j^{(m^\star)}}$.

For any $j \neq j^\star(z_{m^\star})$ in $\ttset_{-t}$:
\begin{align*}
\frac{\mathbb{E}[\Eb_{j^\star} \mid \mathcal{F}_s]}{\mathbb{E}[\Eb_j \mid \mathcal{F}_s]} 
&> e^{u_{j^\star}^{(m^\star)} - u_j^{(m^\star)}} \frac{1 - \varepsilon_\gamma}{1 - \varepsilon_\gamma + \varepsilon_\gamma R_{\max}} \geq e^\gamma \frac{1 - \varepsilon_\gamma}{1 - \varepsilon_\gamma + \varepsilon_\gamma R_{\max}},
\end{align*}
and this ratio exceeds 1 by definition of $\varepsilon_\gamma$, proving our claim.

From Assumption \ref{assumption:dgp}, conditional on $Z = z_{m^\star}$, we have $\log \Eb_j = u_j^{(m^\star)} + \xi_j$ where $\xi_j \sim \text{Gumbel}(0, 2\sigma)$. The ranking follows a Plackett--Luce distribution with weights $\{\theta_j^{(m^\star)}\}$~\citep{yellott1977relationship,ma2021learning}. Under the PL sequential sampling view, the probability that $j^\star(z_{m^\star})$ is not selected in the first $s$ draws is bounded as
$\prob\left( \rank_{\ttset_{-s}}(j^\star) > t-s \mid Z = z_{m^\star} \right) \leq e^{-s p^\star(z_{m^\star})} \leq e^{-s \rho_{\mathrm{PL}}(d-s)}$,
where $p^\star(z_{m^\star}) := \theta_{j^\star}^{(m^\star)} / \sum_{\ell \in J_s} \theta_\ell^{(m^\star)} \geq \rho_{\mathrm{PL}}(d-s)$. This bound does not depend on $z_{m^\star}$, so we may average over $Z$.

Finally, decomposing $\bad(s,t)$ by latent states and conditioning on $E_{\mathrm{post}}(s,z_{m^\star})$ vs.\ its complement:
\begin{align*}
\prob(\bad(s,t)) \leq \prob(E_{\mathrm{post}}(s)) + e^{-k \rho_{\mathrm{PL}}(d-t)}.
\end{align*}
\end{proof}

\subsection{Technical Lemma 4: Posterior Concentration}

Recall that given $Z = z_m$, we have $X_j \sim \text{Gumbel}(\eta_j + \beta_j z_m, \sigma)$ with density $f_{j,m}$. Fix $\lambda \in (0, 1)$ and define the Chernoff coefficient \citep{nielsen2011chernoff} between states $m$ and $m'$ for feature $j$ as
$c_j(m, m'; \lambda) := \int f_{j,m}(x)^{1-\lambda} f_{j,m'}(x)^\lambda \, dx$.
For two Gumbel distributions with locations $\mu, \mu'$ and common scale $\sigma$, this has the closed form
$c(\mu, \mu'; \lambda) = \frac{e^{-\lambda \Delta}}{(1 - \lambda) + \lambda e^{-\Delta}}$, where $\Delta := \frac{\mu - \mu'}{\sigma}$.

Let $\rho_\lambda := \max_{j \in [d]} \max_{m \neq m'} c_j(m, m'; \lambda) < 1$ and $p_{\min} := \min_{m \in [M]} p_m > 0$.

\begin{lemma}
\label{technical_lemma:posterior-concentration}
For any value of $Z = z_{m^\star}$, any $\varepsilon \in (0, 1)$, any $s \geq 1$, and $\lambda \geq 1$, 
\[
\prob\left( \alpha_{m^\star}(\mathcal{F}_s) \leq 1 - \varepsilon \mid Z=z_{m^\star} \right) 
\leq (M - 1) \left( \frac{1 - \varepsilon}{\varepsilon p_{\min}} \right)^\lambda \rho_\lambda^s.
\]
In particular, setting $\varepsilon = \varepsilon_\gamma$, we have
$\prob\big(E_{\mathrm{post}}(s,z_{m^\star})\big) \leq C_\lambda \rho_\lambda^s$,
where $C_\lambda := (M - 1) \left( \frac{1 - \varepsilon_\gamma}{\varepsilon_\gamma p_{\min}} \right)^\lambda$.
\end{lemma}

\begin{proof}
    For any state $m \neq m^\star$, define the likelihood ratio
$\kappa_s(m) := \frac{p_m}{p_{m^\star}} \prod_{r=1}^s \frac{f_{i_r, m}(X_{i_r})}{f_{i_r, m^\star}(X_{i_r})}$,
where $X_{i_r} = \log x_{i_r}$. By Bayes' rule,
$\alpha_{m^\star}(\mathcal{F}_s) = \frac{1}{1 + \sum_{m \neq m^\star} \kappa_s(m)}$.
If $\alpha_{m^\star}(\mathcal{F}_s) \leq 1 - \varepsilon$, then $\sum_{m \neq m^\star} \kappa_s(m) \geq \frac{\varepsilon}{1 - \varepsilon}$. By the pigeonhole principle, there exists some $m \neq m^\star$ such that $\kappa_s(m) \geq \frac{\varepsilon}{(1 - \varepsilon)(M - 1)}$.

By a union bound and Markov's inequality with exponent $\lambda \geq 1$:
\begin{align*}
\prob\left( \alpha_{m^\star}(\mathcal{F}_s) \leq 1 - \varepsilon \mid Z=z_{m^\star} \right) 
\leq \sum_{m \neq m^\star} \frac{\mathbb{E}[\kappa_s(m)^\lambda\mid Z=z_{m^\star} ]}{\left( \frac{\varepsilon}{(1-\varepsilon)(M-1)} \right)^\lambda}.
\end{align*}

Since $X_{i_1}, \ldots, X_{i_s}$ are independent draws from $f_{i_1, m^\star}, \ldots, f_{i_s, m^\star}$ given $Z = z_{m^\star}$, each factor satisfies
$\mathbb{E}\left[ \left( \frac{f_{i_r, m}(X_{i_r})}{f_{i_r, m^\star}(X_{i_r})} \right)^\lambda \,\Big|\, Z = z_{m^\star} \right] = c_{i_r}(m^\star, m; \lambda) \leq \rho_\lambda$.
Therefore,
$\mathbb{E}\left[ \kappa_s(m)^\lambda \mid Z = z_{m^\star}\right] \leq p_{\min}^{-\lambda} \rho_\lambda^s$.
Summing over all $M - 1$ wrong states yields the claim. Since the bound does not depend on $z_{m^\star}$, it holds for any latent state.
\end{proof}

\subsection{Technical Lemma 5: Spectral Misalignment Bound}

Define for any subset $\bm{P} \subset \{1,\dots,d\}$:
$\mathcal{M}(\bm{P}):= \trueloss(\bm{P}) - \proxyloss(\bm{P})$
and 
$u(\bm{P}) := \Big[\mathbf{1}_{\{j \notin \bm{P}\}} \|\wb_j\| \xb_j\Big]_{j=1}^d$.

\begin{lemma}
\label{tech_lemma_miss}
    For any two sets $\bm{P},\bm{P}'$, let $\bm{U}(\bm{P},\bm{P}') := \big(\bm{P} \cap \bm{P}'\big)^c$ and $\Gamma(\Pb, \Pb') = \Pb \Delta \Pb'$. We have 
    \[
|\Mcal(\bm{P}) - \Mcal(\bm{P}') | \leq C' \left(\sqrt{\frac{|\bm{U}(\bm{P},\bm{P}')|}{p_{\text{eff}}} } + \frac{|\bm{U}(\bm{P},\bm{P}')|}{{p}_{\text{eff}}}\right) \sqrt{\sum_{j \in \bm{U}(\bm{P},\bm{P}')} \Eb_j}\sqrt{\sum_{j \in \Gamma(\bm{P},\bm{P}')} \Eb_j}
\] 
\end{lemma}
\begin{proof}
Let $B := \big[\frac{\wb_j}{\|\wb_j\|}\big]_{j=1}^d$ and define $H:= B^TB - I$. We have
$\Mcal(\bm{P}) - \Mcal(\bm{P}') = (u(\bm{P}) - u(\bm{P}'))^\top H_U (u(\bm{P}) + u(\bm{P}'))$
where $H_U := B_U^\top B_U - I$ is the restriction of $H$ to $\bm{U}(\bm{P},\bm{P}')$. Therefore:
\begin{align*}
   |\Mcal(\bm{P}) - \Mcal(\bm{P}') | &\leq \|H_U\|_\text{op} \|u(\bm{P}) - u(\bm{P}')\|_2 \|u(\bm{P}) + u(\bm{P}')\|_2  \\
&\leq \|H_U\|_\text{op} \|u(\bm{P}) - u(\bm{P}')\|_2(\|u(\bm{P})\|_2 + \|u(\bm{P}')\|_2 ).
\end{align*}
Note that $|\bm{U}(\bm{P},\bm{P}')| = d-t+k_t$, and $\|u(\bm{P}) - u(\bm{P}')\|_2^2 = \sum_{j \in {\Gamma}(\bm{P},\bm{P}')} \|w_j\|_2^2 x_j^2 = \sum_{j \in \bm{U}(\bm{P},\bm{P}')} V_j$. Using $\|u(\bm{P}) + u(\Pb')\|_2^2 \leq 4 \sum_{j \in U(\Pb, \Pb')} \Eb_j$ and Assumption \ref{asm:isometry} concludes the proof.
\end{proof}

\newpage

\section{Experimental Details}
\label{appendix:exp_details}

\subsection{Datasets}
\label{appendix:datasets}

We evaluate \ours and all baselines on six datasets spanning seven downstream tasks. Table~\ref{tab:app-datasets} provides a summary. Below we describe each dataset in detail, including the downstream task, preprocessing choices, and the specific aspects that make it relevant for evaluating budgeted feature acquisition.

\begin{table}[t]
    \centering
     \caption{Overview of datasets used for experiments.}
     \adjustbox{max width=\textwidth}{
    \begin{tabular}{lccccl}
   \toprule
        \textbf{Dataset} &\textbf{Task}& \textbf{Train/Val/Test} & \textbf{\# classes}&\# \textbf{views} &  \textbf{Source}\\
        \midrule
        Synthetic & \begin{tabular}{@{}c@{}}Random isotropic \\ regression \end{tabular} & 200/-/200 & ------ & 64 & ------ \\ \hline
         \multirow{2}{*}{NSCLC patient cohort} & Cancer grading & 1.6k/400/400& 3 & 40 & \citet{cords2024cancer} \\
         & Cancer subtyping & & 2 & 40 & \citet{cords2024cancer}\\ \hline
        {\hp} &\begin{tabular}{@{}c@{}}Cell type \\ annotation \end{tabular} & 9.5k/1k/4.2k & 14 & 200 &  \citet{dingtabula} \\ \hline
         MiniBooNE & \begin{tabular}{@{}c@{}}Particle \\ identification \end{tabular} &65k/5k/6k & 2 & 33 & \citet{miniboone}\\\hline
         MNIST & \begin{tabular}{@{}c@{}}Image \\ classification \end{tabular} & 50k/10k/10k 
         & 10 & 196 & \citet{lecun1998mnist}  \\ \hline
         Fashion-MNIST & \begin{tabular}{@{}c@{}}Image \\ classification \end{tabular} & 50k/10k/10k 
         & 10 & 196 & \citet{xiao2017fmnist}  \\ \hline 
         MMEarth-Bench & \begin{tabular}{@{}c@{}}  Soil pH regression \end{tabular} & 5.7k/1.2k/350 
         & ------ & 12 & \citet{gordon2026mmearth}  \\
         \bottomrule
         
    \end{tabular}
    \label{tab:app-datasets}
}
\end{table}

\subsubsection{Vision Datasets: \mnist and \fmnist}
\mnist~\citep{lecun1998mnist} and \fmnist~\citep{xiao2017fmnist} are standard computer vision benchmarks consisting of images of handwritten digits and Zalando article images, respectively. Both have 10-class classification as the downstream task. Following~\citet{li2021active}, we downsample images from $28 \times 28$ to $14 \times 14$ resolution, treating each pixel as an independently acquirable feature and yielding $d = 196$ acquisition units. This downsampling increases the information density per pixel and makes the acquisition problem more challenging, since fewer pixels suffice to reach a given accuracy. We use a 70:15:15 split (50k/10k/10k). To ensure a controlled comparison, all acquisition methods are initialized from the same top-corner pixel ($\vcenter{\hbox{\fcolorbox{black}{topcornercolor}{\rule{0pt}{0.5pt}\rule{0.5pt}{0pt}}}}$), so differences in accuracy reflect only the quality of the subsequent acquisition strategy.

\subsubsection{Physics: \minib}
The MiniBooNE dataset~\citep{miniboone} originates from the Fermilab MiniBooNE neutrino experiment. The goal is binary classification: distinguishing electron-neutrino signal events from muon-neutrino background events. The dataset contains $d = 33$ reconstructed physics variables derived from detector signals, including Cherenkov ring characteristics, event energy, and track reconstruction outputs. In practice, computing these reconstructed variables from raw detector data is expensive and the variables exhibit substantial mutual correlations, making it undesirable to rely on the full feature set at test time. We adopt the preprocessing protocol of~\citet{norcliffe2025stochastic}, which standardizes features and removes duplicate events, and use a 70:15:15 split (65k/5k/6k). This dataset is particularly relevant for evaluating \ours because the strong inter-feature correlations create a setting where greedy, myopic acquisition can be suboptimal: selecting the individually most informative variable first may be redundant with the second-most informative one, motivating the long-horizon planning that \ours provides.

\subsubsection{Single-Cell Transcriptomics: \hp}
The \hp dataset~\citep{dingtabula} is a curated single-cell RNA-sequencing collection from human pancreatic tissue with 14 cell type labels spanning endocrine, exocrine, and immune populations. This dataset is widely used as a benchmark in the single-cell foundation model literature. From the 1,200 genes available, we retain the top 200 most variable genes using Cell Ranger~\citep{zheng2017massively} and ScanPy~\citep{wolf2018scanpy}, following standard practice for reducing noise from uninformative transcripts while preserving biologically meaningful variation. We use the preprocessing, gene filtering, and data splits provided by~\citep{dingtabula} (9.5k/1k/4.2k). Cell-type annotation serves as a representative downstream task: it is the most common evaluation objective in single-cell analysis and requires the acquired gene subset to preserve fine-grained phenotypic distinctions between closely related cell populations. Crucially, the \ours policy is trained without any access to cell-type labels; the annotation task is used purely for post-hoc evaluation of the acquired gene panels.

\subsubsection{Spatial Proteomics: NSCLC Patient Cohort}
The NSCLC clinical cohort~\citep{cords2024cancer} consists of spatial proteomics data acquired via imaging mass cytometry (IMC) at $1\,\mu\mathrm{m}$/px resolution with $d = 40$ semantically meaningful protein markers. The marker panel includes tumor markers (e.g., \texttt{panCK}), proliferation markers (e.g., \texttt{Ki-67}), immune cell markers (e.g., \texttt{CD4}, \texttt{CD8}, \texttt{CD68}), stromal markers, and vascular markers, providing a comprehensive view of the tumor microenvironment. Clinical annotations include cancer subtyping (2 classes: adenocarcinoma vs.\ squamous cell carcinoma) and cancer grading (3 classes), the latter being substantially more challenging due to subtler morphological and molecular differences between grades. We adopt the preprocessing pipeline described in~\citep{wenckstern2025ai}, which includes cell segmentation, marker normalization, and quality filtering. The experimental setup is inherently sequential: \ours iteratively selects one protein marker at a time, the VirTues backbone updates its tissue embedding given the newly observed marker, and the updated embedding informs the next acquisition decision. This mirrors the real-world workflow in multiplexed imaging, where tissue sections are characterized through successive rounds of antibody staining and each additional round carries a cost. The acquisition budget is set to $T = 10$ markers (25\% of the full panel), reflecting realistic constraints on antibody panel size in clinical multiplex imaging.

\subsubsection{Geospatial Earth Observations: MMEarth-Bench}
\label{appendix:geospatial_details}

The MMEarth-Bench dataset~\citep{gordon2026mmearth} is a global multimodal Earth-observation benchmark to evaluate models under heterogeneous sensing conditions. Each location is represented by 12 aligned geospatial views, including optical satellite imagery, radar measurements, terrain and elevation products, land-cover maps, canopy-height information, meteorological variables, and geolocation metadata. It provides a multimodal acquisition setting: each view corresponds to a distinct sensing product or environmental data source that may differ in availability, cost, spatial resolution, and information content. We treat each modality group as one acquisition unit and pair \ours with a frozen multimodal geospatial representation model trained on the full set of available views. We evaluate downstream transfer on soil pH regression, reporting $R^2$
 as a function of the acquired modality budget.

\subsubsection{Synthetic Dataset}
\label{appendix:synthetic_details}
\begin{wrapfigure}{r}{0.5\textwidth}
    \centering
    \includegraphics[width=0.9\linewidth]{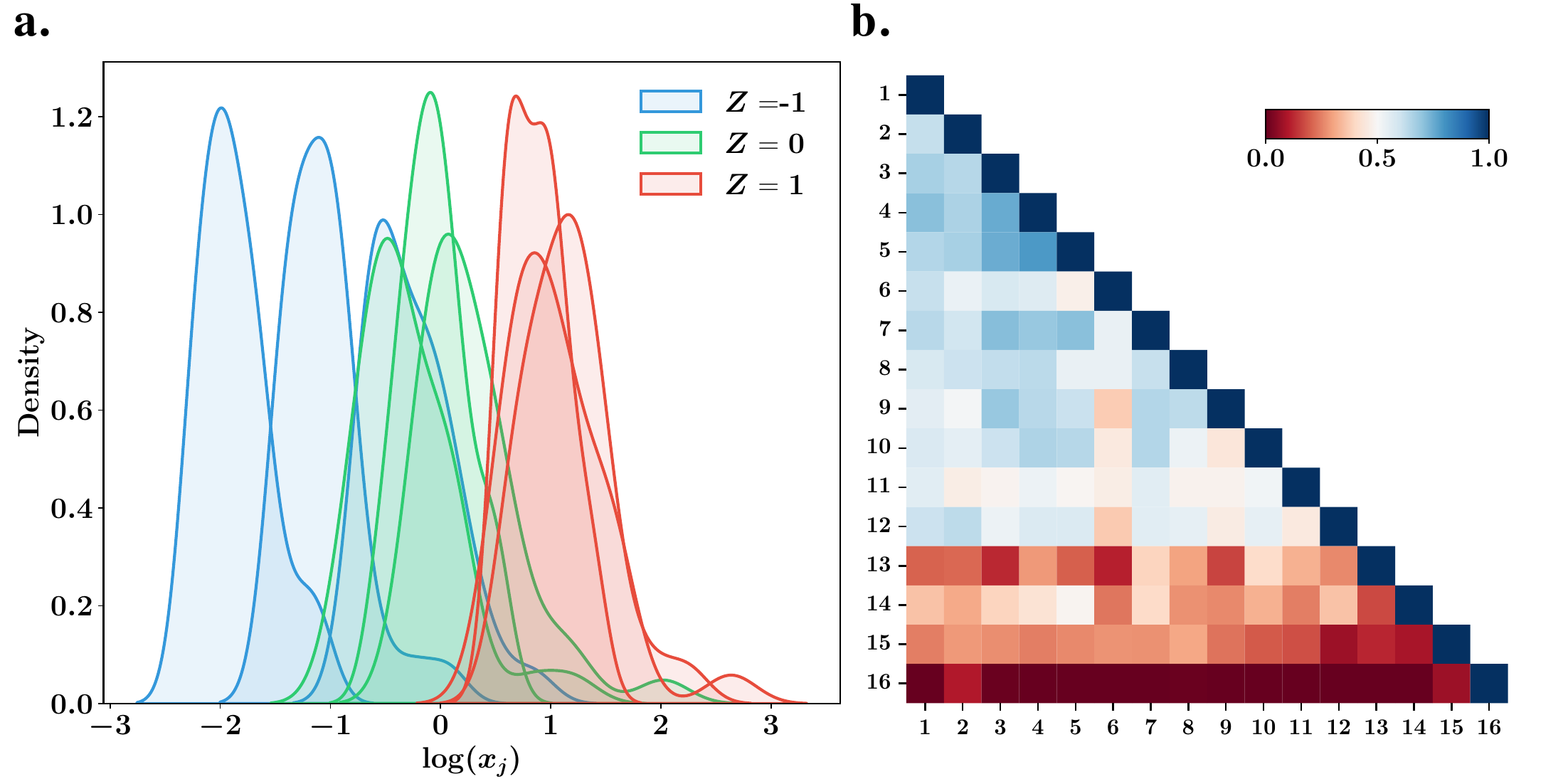}
    \caption{\textbf{Overview of the Latent Gumbel Correlation Model.} \textbf{a.} Distribution of $\log(\xb_j)$ for 3 equiprobable latent states $Z$ with support $\{-1,0,1\}$. \textbf{b.} Correlation of the features under the gradient correlation family described in \S~\ref{appendix:synthetic_details}.}
    \label{fig:app-cor}
\end{wrapfigure}
We generate synthetic data under the data generating process (DGP) of Assumption~\ref{assumption:dgp}. We consider a gradient correlation family shared by a common latent state $Z$, where $Z \sim \text{Discrete}\left(\{z_m\}_{m=1}^M, \{p_m\}_{m=1}^M\right)$,
$\varepsilon_j \sim \text{Gumbel}(0,\sigma)$ being independent noise. Each feature is generated as
$X_j \sim \eta_j + \beta_j Z + \varepsilon_j$, where $\eta_j$ and $\beta_j$ are deterministic coefficients, and we develop a gradient-like correlation structure by letting
\[
\beta_j := \beta_{\max} \left(1 - \alpha \frac{j-1}{d-1}\right),
\]
and set $x_j = e^{X_j}$. In Fig.~\ref{fig:app-cor}a, we show the data distribution when the support of $Z$ has cardinality 3, and $z_m \in \{-1,0,1\}$. Here $\beta_{\max}$ is the maximum strength and $\alpha \in [0,1]$ controls the decay rate. Under this construction, early features exhibit strong mutual correlations through $Z$, while later features are increasingly dominated by independent noise. The induced correlation structure is:
\[
\text{Corr}(X_i, X_j) = \frac{\beta_i \beta_j \text{Var}(Z)}{\sqrt{\left(\beta_j^2 \text{Var}(Z) + \frac{\pi^2\sigma^2}{6} \right)\left(\beta_k^2 \text{Var}(Z) + \frac{\pi^2\sigma^2}{6} \right)}}.
\]

\paragraph{Regret bound experiment (Fig.~\ref{fig:croissant}} We set $p = 25, d = 32$ (i.e., $\Wb \in \RR^{p\times d}$), $\sigma = 0.2$, $M=5$ with $z_m \in \{-1, -1/2, 0, 1/2, 1\}$ and $p_m = 0.2 \; \forall m$. We consider $\eta_j = \frac{j-1}{d-1}-\frac{1}{2}$ for $j \in 1, \dots, d$. For the gradient correlation family, we set $\beta_{\max} = 1.5$ and $\alpha = 0.9$.

\paragraph{Random isotropic regression experiment (Fig.~\ref{fig:control}).} We increase the acquisition horizon by setting $d=64$ and $p=50$, keeping all other hyperparameters fixed. We sample the ground-truth regression weights as $\Wb^* \sim \mathcal{N}(0, \tfrac{1}{d} I)$ and generate labels as $\yb = \Wb^* \xb + \bm{\varepsilon'}$ where $\bm{\varepsilon}' \in \mathbb{R}^p$ has i.i.d. entries sampled from $\mathrm{Gumbel}(0,1)$.

\subsection{Architectures and Backbone Details}
\label{appendix:architectures}

\subsubsection{Masked Autoencoder (\mnist, \fmnist, \minib)}
For image and tabular data, we train a masked autoencoder via self-supervised learning on the full dataset. Given an input $x$, a random subset of features is revealed via a binary mask $m \in \{0,1\}^d$, and the model learns an encoder $z = f_\theta(x \odot m, m)$ and decoder $\hat{x} = g_\theta(z)$ to reconstruct the full input using an MSE objective. The encoder receives both the masked input and the mask itself, allowing it to distinguish between zero-valued features and unobserved ones. The masking ratio is sampled uniformly during training to expose the model to a wide range of partial observations. After pretraining, the encoder is frozen and used as the representation backbone for \ours. For downstream evaluation, a linear probe is trained on $z$ for classification under random masking to simulate partial observations; this ensures that the probe is calibrated to the type of partial inputs encountered during sequential acquisition. The same frozen encoder is shared across all baselines and \ours.

\subsubsection{TABULA Foundation Model}
For single-cell RNA-seq, we use the pretrained TABULA foundation model~\citep{dingtabula}, a self-supervised transformer designed for single-cell transcriptomic data. TABULA is trained on large-scale single-cell atlases and can embed partially observed gene expression profiles into cell-level representations. When a subset of genes is unobserved, TABULA treats these entries as missing and produces an embedding conditioned only on the available genes. This native support for partial observations makes it well-suited as a backbone for \ours. The downstream task of cell-type annotation is evaluated via a linear probe trained on top of the cell embeddings. We use TABULA without any fine-tuning; the model weights are frozen throughout both policy learning and evaluation.

\subsubsection{VirTues Foundation Model}
For the NSCLC patient cohort, we deploy \ours with the VirTues spatial proteomics foundation model~\citep{wenckstern2025ai}. VirTues is a graph-based transformer that operates on cell-level protein expression profiles within tissue sections, producing both cell-level and tissue-level representations. Tissue-level representations are obtained by averaging the cell-level summary tokens across all cells in the tissue sample. When a protein marker is unobserved, its expression value is treated as missing in the cell-level input, and VirTues produces an embedding conditioned on the available markers. Linear probes are trained for cancer subtyping and grading using the tissue-level embeddings. Importantly, the same \ours policy is used for both downstream tasks: the acquisition order is determined once by the policy and then evaluated on both tasks, demonstrating task-agnostic panel design.

\subsubsection{GeoSpatial Encoder}
For MMEarth-Bench, a masked encoder is used. Concretely, each of the 12 Earth-observation modalities is first converted to a modality-level vector by global average pooling for spatial inputs and identity for tile-level inputs, then projected to a shared dimension and multiplied by a binary modality mask indicating which modalities are observed. The masked per-modality projections are concatenated
  together with the mask itself and passed through an MLP encoder to produce a latent representation $z$. This encoder is trained  under random modality masking, so it learns to make reconstructions from arbitrary partial modality subsets while
  explicitly distinguishing unobserved modalities from zero-valued features. 

\paragraph{Backbone-agnosticism.} Across all settings, both \ours and the baselines have access only to the pretrained models during the feature selection phase. Linear probes are used exclusively for downstream evaluation after feature acquisition. \ours interfaces with each backbone only through the partial embedding $\zb(S_t)$ and the full-information target $\zb^\star$; no architecture-specific modifications are needed. This design means that \ours can be deployed with any new backbone simply by providing the embedding function, without rewriting any acquisition logic.

\subsection{Evaluation Protocol}
\label{appendix:evaluation}

Each dataset is partitioned into training, validation, and test splits. For image and tabular tasks, we use a 70:15:15 split. For cell-type annotation, we adopt the predefined splits released with TABULA~\citep{dingtabula}, and for cancer subtyping and grading, we follow the test splits defined in~\citep{wenckstern2025ai}.

\paragraph{Acquisition curves.} The primary evaluation metric is the acquisition curve: for each method, we plot downstream performance (MSE or classification accuracy) as a function of the number of features acquired $t \in \intset{1}{d}$. This allows us to compare methods across the full range of acquisition budgets and identify regimes where differences are most pronounced. In practice, the early regime (5--25\% of features) is of greatest interest, since this is where cost savings are largest.

\paragraph{Linear probes.} For all real-world datasets, downstream performance is evaluated by training a linear probe on the partial embeddings $\zb(S_t)$ produced by the frozen backbone. A single linear probe is trained per downstream task using random masking augmentation and held fixed across acquisition budgets. This ensures that the evaluation is task-specific (it depends on the downstream labels) while the acquisition policy remains task-agnostic.

\paragraph{Reproducibility.} All experiments are run with three random seeds for both the \ours policy and the downstream evaluation. The \texttt{Random} baseline is additionally averaged over five independent runs per seed to reduce variance. Error bands in all figures correspond to standard error across seeds.

\subsection{Baselines}
\label{appendix:baselines}

\newcommand{\bsoracle}{\texttt{Oracle}\xspace}
\newcommand{\bsrandom}{\texttt{Random}\xspace}
\newcommand{\bsprior}{\texttt{Prior}\xspace}
\newcommand{\bsgir}{\texttt{GIR}\xspace}
\newcommand{\bsesmmax}{\texttt{ESM-Max}\xspace}
\newcommand{\bsesmmin}{\texttt{ESM-Min}\xspace}
\newcommand{\bslap}{\texttt{Laplacian Score}\xspace}
\newcommand{\bsspec}{\textsc{spec}-$\varphi_3$\xspace}
\newcommand{\bsmcfs}{\texttt{MCFS}\xspace}
\newcommand{\bsudfs}{\texttt{UDFS}\xspace}
\newcommand{\bssvd}{\texttt{SVD-Ent}\xspace}
\newcommand{\bsvar}{\texttt{Variance}\xspace}
\newcommand{\bsgsmrl}{\texttt{GSMRL-AIR}\xspace}
\newcommand{\bsalg}{\texttt{ALG}\xspace}
\newcommand{\bscig}{\texttt{CIG}\xspace}

We compare \ours against a total of 15 baselines spanning synthetic~(S), image~(I), tabular~(T), single-cell RNA-seq~(C), and multiplex imaging~(M) datasets. Below, we describe each baseline and indicate the settings in which it is evaluated.

\subsubsection{Oracle and Random Baselines}

\xhdr{\bsoracle (S, I, T).}
A one-step myopic oracle that greedily selects the next feature maximizing the immediate reduction in distance between the current embedding and the full (gold) embedding. This policy is anti-causal: it has access to unobserved features and corresponds to the oracle analyzed in our theoretical results (\S\ref{section:theory}). It serves as an upper bound on greedy acquisition performance but is not deployable at test time.

\xhdr{\bsrandom (S, I, T, C, M).}
Selects the next view uniformly at random from the set of unobserved features. Provides a minimal reference point for all experiments.

\subsubsection{Model-Driven Heuristic Baselines}

\xhdr{\bsprior (S).}
Specific to the synthetic setting. Ranks features according to their unconditional expected contribution $\mathbb{E}[E_j]$, which under Assumption~\ref{assumption:dgp} admits the closed form:
\[
\EE[\Eb_j] = \|\mathbf{w}_j\|_2^2 \Gamma(1-2\sigma)e^{2\eta_j} \sum_{m=1}^M p_m e^{2\beta_j z_m} \propto \|\mathbf{w}_j\|_2^2 e^{2\eta_j} \sum_{m=1}^M p_m e^{2\beta_j z_m}.
\]
Unlike \ours, this baseline does not condition on observed features and therefore captures only population-level feature importance.

\xhdr{\bsgir (I, T, S, M).}
Global Impact Ranking is a practical extension of \bsprior that does not rely on a known generative model. For each feature $i$, a validation-set score $s_i := \|\zb_{\text{full}} - \zb_{-i}\|_2$ is computed, where $\zb_{\mathrm{full}}$ denotes the self-supervised embedding obtained using all features and $\zb_{-i}$ denotes the embedding obtained when feature $i$ is withheld. This induces a fixed, instance-independent acquisition order that captures population-level importance but not individual-level adaptivity.

\xhdr{\bsesmmax, \bsesmmin (M).}
For multiplex imaging only. Protein markers are embedded using ESM2~\citep{lin2023evolutionary}. For each marker, pairwise Euclidean distances are computed in the ESM2 embedding space. Markers are ranked by either minimum (\bsesmmin) or maximum (\bsesmmax) distance, yielding fixed acquisition orders based on semantic similarity or dissimilarity, following~\citep{jain2025testtime}.

\subsubsection{Unsupervised Feature Selection Baselines}

The following methods were originally designed to identify informative feature \emph{subsets} given full feature access~\citep{solorio2020review}. We adapt them to the sequential acquisition setting by using their induced feature rankings as acquisition orders. While not designed for instance-adaptive or sequential settings, they provide strong unsupervised baselines for ranking-based selection. All scores are computed on the validation set and induce a fixed acquisition order applied at test time.

\xhdr{\bsvar (I, T, S).}
Ranks features by their marginal variance over the validation set, selecting in decreasing order of variance.

\xhdr{\bsspec (I, T).}
Ranks features based on alignment with the leading $k\!-\!1$ eigenvectors of the graph Laplacian, where $k$ is the number of clusters~\citep{zhao2007spectral}. Features with stronger alignment to these optimal soft cluster indicators receive higher scores.

\xhdr{\bslap (I, T).}
The Laplacian Score~\citep{he2005laplacian} evaluates features by how well they preserve local geometric structure encoded by a nearest-neighbor graph. For each feature $r$, it computes $L_r = \frac{\widetilde{f}_r^\top L\widetilde{f}_r}{\widetilde{f}_r^\top D \widetilde{f}_r}$, where $L = D - S$ is the graph Laplacian, $S$ is the similarity matrix, and $\widetilde{f}_r$ is the mean-centered weighted feature vector. Features with smaller scores are ranked higher.

\xhdr{\bsmcfs (I, T).}
Multi-Cluster Feature Selection~\citep{cai2010unsupervised} selects features via spectral analysis of the graph Laplacian. It first solves the generalized eigenvalue problem $Ly = \lambda Dy$ to obtain $K$ eigenvectors $Y = [y_1,\ldots,y_K]$. For each $y_k$, it solves an $\ell_1$-regularized regression $\min_{a_k} \|y_k - X^T a_k\|^2 + \beta|a_k|$ to obtain sparse coefficients. Features are ranked by $\texttt{MCFS}(j) = \max_k |a_{k,j}|$ and selected in decreasing order.

\xhdr{\bsudfs (I, T).}
Unsupervised Discriminative Feature Selection~\citep{yang2011udfs} learns a sparse projection matrix $W \in \mathbb{R}^{d \times c}$ by minimizing local discriminative scores with $\ell_{2,1}$-norm regularization: $\min_{W^T W = I} \text{Tr}(W^T M W) + \gamma \|W\|_{2,1}$, where $M$ encodes local neighborhood structure via $k$-nearest neighbors. Features are ranked by the row norms $\|w_i\|_2$ of $W$ in descending order.

\xhdr{\bssvd (I, T).}
SVD-Entropy~\citep{varshavsky2006novel} ranks features by their contribution to SVD-entropy. For each feature $i$, its contribution is $\text{CE}_i = E(A) - E(A_{-i})$, where $E = -\frac{1}{\log N} \sum_{j=1}^N V_j \log V_j$ and $V_j = s_j^2 / \sum_k s_k^2$ are normalized squared singular values. Features with larger $\mathrm{CE}_i$ are ranked higher.

\subsubsection{Adaptive Label-Free Baselines}

The baselines above induce fixed, instance-independent acquisition orders. To compare against policies that condition acquisition on the current instance, we additionally evaluate the following label-free adaptive baselines on \minib and \hp (\S~\ref{appendix:adaptive_baselines}). As for \ours, these methods observe only the currently acquired values and the acquisition mask at deployment, and never access the full-view embedding or downstream labels.

\xhdr{\bsgsmrl (T, C).}
An instance-adaptive RL method based on acquisition with generative surrogate models~\citep{li2021active}, adapted to the label-free setting. The policy conditions its action selection on the per-sample conditional imputations of unobserved features produced by a generative surrogate, so that acquisition decisions depend on the values observed for each instance.

\xhdr{\bsalg (T, C).}
Adaptive Latent Gain is a learned regressor that predicts, for each unobserved view, the conditional one-step improvement in representation recovery directly from the current partial observation, and greedily acquires the view with the largest predicted gain. \bsalg shares the reward and state representation of \ours, but optimizes only the immediate gain rather than a long-horizon objective, and therefore isolates the contribution of sequential planning.

\xhdr{\bscig (T).}
Cluster Information Gain fits a Gaussian mixture model to the training data and scores each unobserved feature by the mutual information between the feature and the mixture component under the posterior given the currently observed features. The feature with the largest information gain is acquired next. Evaluated on \minib only.

\subsubsection{Summary Table}

\begin{table}[htb!]
    \centering
    \footnotesize   
    \caption{Summary description of implemented baselines. Methods marked with $(\dagger)$ rely on unsupervised feature selection and have oracle access to all features at ranking time. Methods marked with a $(\star)$ are only deployed on the cancer grading and typing datasets, since they rely on a protein sequence foundation model \citep{lin2023evolutionary}. Methods marked with $(\ddagger)$ are label-free adaptive baselines that select instance-specific feature sets through continuous updates of their selection policies.}
    \begin{tabular}{ccccc}
    \toprule
        \textbf{Name} & \textbf{Model Driven?} &\textbf{Selection Criterion} & \textbf{Adaptive?} & \textbf{Source} \\
        \toprule
         \bsoracle & \ding{51} & 1-step Oracle knowledge & \ding{51}& ----\\
         \bsgir & \ding{51} & Average contribution to latent state &   \ding{55} & ---- \\
         \bsrandom & \ding{55} &  Uniform distribution over unseen features & \ding{55} & ---- \\
         \bslap $^\dagger$ & \ding{55} & Locality preservation & \ding{55}&\citet{he2005laplacian}\\
         \bsspec $^\dagger$ & \ding{55} & Graph-based ranking & \ding{55}& \citet{zhao2007spectral}\\
         \bsudfs $^\dagger$ & \ding{55} & Local discriminative structure & \ding{55} & \citet{yang2011udfs}\\
         \bsmcfs $^\dagger$ & \ding{55} & Multi-cluster structure & \ding{55} &\citet{cai2010unsupervised}\\
         \bssvd $^\dagger$ & \ding{55} & SVD-entropy contribution & \ding{55} &\citet{varshavsky2006novel}\\
         \bsvar $^\dagger$ & \ding{55} & Maximum variance & \ding{55} & ---- \\
         \bsesmmax ($\star$)& \ding{51}&ESM2-based & \ding{55}& \citet{lin2023evolutionary}\\
         \bsesmmin ($\star$)& \ding{51}&ESM2-based & \ding{55} & \citet{lin2023evolutionary}\\
         \bsgsmrl ($\ddagger$) & \ding{55} & Generative surrogate imputations & \ding{51} & \citet{li2021active}\\
         \bsalg ($\ddagger$) & \ding{51} & Predicted one-step latent improvement & \ding{51} & ---- \\
         \bscig ($\ddagger$) & \ding{55} & GMM mutual information & \ding{51} & ---- \\
         \midrule 
          \ours & \ding{51} & \begin{tabular}{@{}c@{}}Learned policy via \\  expected latent improvement \end{tabular} & \ding{51} & This paper \\
         \bottomrule
    \end{tabular}
    \vspace{0.2cm}
    \label{tab:summary_baselines}
\end{table}

\subsection{Implementation Details of \ours}
\label{appendix:implementation}

We cast sequential feature acquisition in \ours as a finite-horizon decision problem over partial observations, closely aligned with the conditional surrogate objective analyzed in \S~\ref{section:theory}. The goal of the agent is to sequentially select features so as to minimize the discrepancy between the latent representation induced by the acquired subset and the full-information latent representation of a frozen self-supervised model.

\subsubsection{MDP Formulation}

Let $d$ denote the total number of features and $T$ the acquisition budget. An episode corresponds to sequentially selecting $T$ features without replacement. At step $t$, the agent has observed a subset $S_t \subseteq \{1,\dots,d\}$. We define a finite-horizon Markov decision process $(\mathcal{S},\mathcal{A},P,r)$ as follows.

\paragraph{State.}
The observable state summarizes the acquisition history and the induced latent representation:
\[
s_t := \left(S_t, \zb(S_t), t/T, \|\zb(S_t)\|_2, \xb \odot \bm{m}_t \right)
\]
where $\bm{m}_t \in \{0,1\}^d$ is the binary acquisition mask corresponding to $S_t$, and $\zb(S_t)$ denotes the latent embedding produced by the frozen self-supervised encoder given the partially observed input. This state representation renders the process Markov with respect to the surrogate objective.

\paragraph{Action.}
At each step, the agent selects a feature index $a_t \in \Acal_t := \{1, \dots, d\} \setminus S_t$, corresponding to acquiring a previously unobserved feature. We apply invalid action masking~\citep{huang2020closer}, assigning zero probability to actions in $S_t$ during policy evaluation and optimization.

\paragraph{Transition.}
Transitions are deterministic: selecting $a_t = j$ yields $S_{t+1} = S_t \cup \{j\}$ and updates the latent embedding to $\zb(S_{t+1})$.

\subsubsection{Reward Design}

Let $\zb^*$ denote the latent embedding obtained from the full feature set. Consistent with the surrogate loss analyzed in \S~\ref{section:theory}, we define the instantaneous reward as the normalized reduction in representation error:
\[
r_t = (1 + \mathbbm{1}\{t < t_0\}\lambda) \left(\frac{\|\zb^* - \zb(S)\|_2 - \|\zb^* - \zb(S \cup \{j\})\|_2}{\|\zb^* - \zb(S_0)\|_2}\right),
\]
where $\lambda > 0$ is the gain factor and $S_0$ denotes the singleton initial feature set. This reward is a shaped version of the marginal gain $\Delta(a_t \mid S_t) = \|\zb^*-\zb(S_t)\|_2-\|\zb^*-\zb(S_t\cup\{a_t\})\|_2$, which serves as a proxy for the conditional expected contribution $\mathbb{E}[E_{a_t} \mid \mathcal{F}_t]$ appearing in our theoretical analysis. The step-dependent amplification factor $(1 + \lambda \mathbbm{1}\{t < t_0\})$ encourages refinement in early acquisition stages, where marginal reductions are typically largest.

\subsubsection{Policy Learning}

We learn a stochastic policy $\pi_\theta(a_t \mid s_t)$ using PPO~\citep{schulman2017proximal} with invalid action masking~\citep{huang2020closer}. During both training and inference, the policy has access only to the frozen self-supervised encoder and the induced latent representations. Downstream task labels and linear probes are used exclusively for post-hoc evaluation and do not influence the policy or reward.

\subsubsection{Hyperparameters}

We implement \ours using the \texttt{Stable-Baselines3 (SB3)}~\citep{raffin2021stable} framework with \texttt{sb3-contrib} for support of masked action spaces. Unless mentioned, all PPO parameters are default of SB3. All experiments were conducted on NVIDIA A100 GPUs (80GB) with 128 AMD EPYC 7543 32-core processors. Table~\ref{tab:hyperparams} summarizes the key hyperparameters across datasets.

\begin{table}[t]
    \centering
    \caption{Hyperparameters used for \ours across datasets. The remaining hyperparameters follow the defaults in \texttt{Stable-Baselines3}.}
    \resizebox{\textwidth}{!}{\begin{tabular}{lcccccc}
    \toprule
    \textbf{Parameter} & \textbf{\mnist/\fmnist} & \textbf{\minib} & \textbf{\hp} & \textbf{NSCLC} & \textbf{MMEarth-Bench} & \textbf{Synthetic} \\
    \midrule
    Acquisition budget $T$ & 10 & 7 & 25 & 10 & 5 & varies \\
    Step threshold $t_0$ & $T/2$ & $T/2$ & $T/2$ & $T/2$ & $T/2$ &  $T/2$ \\
    Gain factor $\lambda$ & 1.5 & 1.5 & 1.5 & 1.5 & 1.5 & 1.5 \\
    Discount factor $\gamma$ & 0.98 & 0.98 & 0.99 & 0.98 & 0.98 & 0.98 \\
    Training steps & 300k & 300k & 500k & 300k & 200k & 300k \\
    \bottomrule
    \end{tabular}}
    \label{tab:hyperparams}
\end{table}

The step threshold $t_0$ is uniformly set to half of the episode length $T$. Across experiments, the acquisition budget $T$ is chosen to lie between 5\% and 25\% of the total number of available features, depending on the dataset. We set $\gamma = 0.98$ for small-horizon acquisition and $\gamma = 0.99$ for long-horizon (\hp). For datasets with fewer features, we train for 300,000 steps; for larger feature spaces, we train for 500,000 steps (except for MMEarth-Bench, where we train for 200,000 steps).

\newpage
\section{Additional Experiments}
\label{appendix:additional_experiments}
\begin{wrapfigure}{r}{0.5\textwidth}
    \vspace{-10pt}
    \centering
   \includegraphics[width=0.6\linewidth]{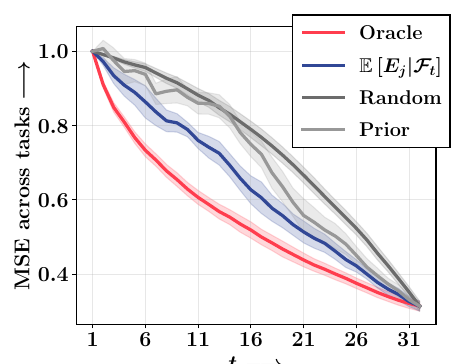}
    \caption{\small Downstream performance of random, prior-based  ($\mathbb{E}[\Eb_j]$), optimal-test-time, and oracle policies on a random isotropic regression task }
    \label{fig:control}
    \vspace{-15pt}
\end{wrapfigure}
\subsection{Controlled Synthetic Experiment}
\looseness -1 We first isolate the core question in a controlled setting with known data-generating process: \emph{how much does conditioning on partial observations help?} For the synthetic data, we consider a randomized isotropic linear regression task with $\Wb \sim \Ncal(0, \frac{1}{d}\Ib_d)$ and $\varepsilon' \sim \text{Gumbel}(\mu',\sigma')$ yielding labels $y = \Wb^\top \xb + \varepsilon'$. We fit $\widehat{{\Wb}}$ via ridge regression on the full feature set (see Assumption~\ref{assumption:dgp}), with labels from isotropic linear regression (details in \S~\ref{appendix:synthetic_details}). Fig.~\ref{fig:control} confirms that the optimal test-time policy consistently outperforms both non-oracle baselines and closely tracks the oracle in early stages, validating the theoretical separation of \S\ref{section:theory} and motivating the long-horizon RL strategy of \ours.

\subsection{Extended Baseline Comparisons}
\label{appendix:extended_baselines}

In vision and tabular settings, features are individually cheap but collectively redundant; a policy that identifies the most informative subset early reduces measurement overhead at test time. For \mnist, \fmnist, and \minib, we train a masked autoencoder and freeze the encoder as \textsc{echo}-$k$'s backbone. For \hp, we use TABULA~\citep{dingtabula} scRNA foundation model. \ours interfaces with it only through $\zb(\cdot)$ and $\zb^\star$, requiring no architecture-specific changes (see \S~\ref{appendix:architectures} for backbone details).

\textbf{Results on \mnist and \fmnist.}
Both datasets are downsampled to $14 \times 14$ following~\citet{li2021active}, yielding $d=196$ pixel-features with 10-class classification as the downstream task. The goal is to classify each image (handwritten digit or clothing item) by sequentially revealing as few pixels as possible. All methods start from the same top-corner pixel (Fig.~\ref{fig:cost}). \ours rapidly concentrates its acquisitions on discriminative image regions, achieving gains of up to 12\% over all non-oracle baselines in the early regime (10--20\% of pixels). No single baseline is consistently second-best: \texttt{Variance} leads on \mnist while \texttt{Random} is surprisingly competitive on \fmnist. 

\textbf{Results on \minib.}
MiniBooNE contains $d=33$ reconstructed physics variables from a neutrino experiment, with binary signal-vs-background classification as the downstream task; these variables are costly to compute and highly correlated (see \S~\ref{appendix:datasets}). \ours outperforms all heuristic baselines and the one-step oracle (Fig.~\ref{fig:bigfig}). We attribute this to the oracle's myopic strategy, which over-selects correlated variables early and suffers diminishing returns, whereas \ours optimizes over a long horizon and learns to exploit complementary information.

\begin{figure}[t]
    \centering
    \includegraphics[width=1\linewidth]{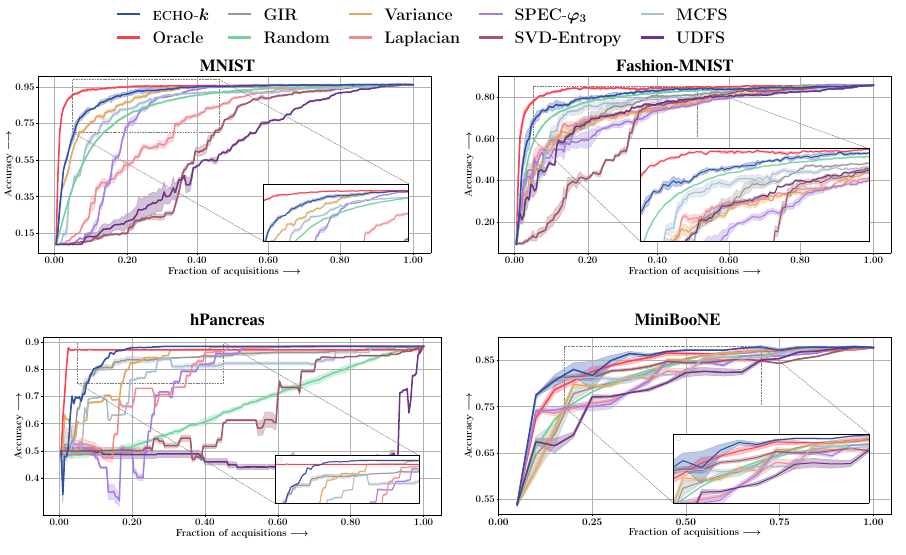}
    \caption{Results of \ours across all baselines for the MNIST (top-left), Fashion-MNIST (top-right), hPancreas (bottom-left), and \minib (bottom-right) datasets.}
    \label{fig:bigfig}
\end{figure}

\subsection{Comparison with Adaptive Acquisition Policies}
\label{appendix:adaptive_baselines}
Most baselines in the main text are fixed acquisition orderings. To verify that the gains of \ours are not solely due to comparison against fixed rankings, we compare against label-free controls that condition acquisition on the current instance and match the information setting of \ours:
\begin{itemize}
    \item \textbf{GSMRL-AIR}, an instance-adaptive RL method that conditions action selection on the per-sample conditional imputations of a generative surrogate~\citep{li2021active}, adapted to the label-free setting.
    \item \textbf{Adaptive Latent Gain (ALG)}, a learned regressor that predicts the conditional one-step representation improvement directly from the current partial observation. ALG shares the reward and state representation of \ours, but not its long-horizon objective.
    \item \textbf{Cluster Information Gain (CIG)}, which scores candidate views by mutual information under a fitted Gaussian mixture model posterior (\minib only).
\end{itemize}
All methods observe only the currently acquired values and the acquisition mask, and never access the full-view embedding or downstream labels at deployment.

\begin{table}[h]
\centering
\setlength{\aboverulesep}{0pt}
\setlength{\belowrulesep}{0pt}
\caption{\small \textbf{Comparison with adaptive acquisition policies.} Accuracy (\%) at different acquisition budgets on \minib and \hp. All policies operate without downstream labels. \legendbox{blue!40} denotes best and \legendbox{pink!40} second best performance.}
\label{tab:adaptive_baselines}
\adjustbox{max width=\textwidth}{
\begin{tabular}{lcccc@{\hspace{\tabcolsep}\hspace{\tabcolsep}}lcccc}
\toprule
\multirow{2}{*}{\textbf{Method}} & \multicolumn{4}{c}{\textbf{\minib}} & \multirow{2}{*}{\textbf{Method}} & \multicolumn{4}{c}{\textbf{\hp}} \\
\arrayrulecolor{lightgray} \cmidrule{2-5} \cmidrule{7-10} \arrayrulecolor{black}
 & \textbf{20\%} & \textbf{40\%} & \textbf{60\%} & \textbf{80\%} & & \textbf{20\%} & \textbf{40\%} & \textbf{60\%} & \textbf{80\%} \\ \midrule
CIG & \secbest{$79.80$} & \secbest{$84.87$} & $86.17$ & $86.97$ & & & & & \\
ALG & $77.50$ & $83.60$ & $86.12$ & \secbest{$87.26$} & ALG & $73.38$ & $80.16$ & $83.12$ & $86.33$ \\
GSMRL-AIR & $79.18$ & $84.78$ & \secbest{$86.40$} & $86.26$ & GSMRL-AIR & \secbest{$77.23$} & \secbest{$81.87$} & \secbest{$84.41$} & \secbest{$86.40$} \\ \arrayrulecolor{lightgray} \hline \arrayrulecolor{black}
\ours & \best{$81.70$} & \best{$86.87$} & \best{$87.60$} & \best{$87.83$} & \ours & \best{$80.81$} & \best{$88.04$} & \best{$88.43$} & \best{$88.45$} \\
\bottomrule
\end{tabular}
}
\end{table}

Table~\ref{tab:adaptive_baselines} shows that \ours outperforms all adaptive controls at every budget on both datasets, including GSMRL-AIR, a fully adaptive RL baseline operating under the same label-free constraints. The comparison with ALG is particularly informative: both methods use representation improvement to guide acquisition, but \ours optimizes gains over a longer horizon. This indicates that the advantage of \ours comes from long-horizon planning rather than from the recovery signal alone.

\subsection{Comparison with Supervised Acquisition}
\label{appendix:supervised}
Table~\ref{tab:supervised_full} reports the full comparison of \ours with five supervised acquisition methods on \minib: DIME~\citep{gadgil2023estimating}, GDFS~\citep{covert2023learning}, SEFA~\citep{norcliffe2025stochastic}, GSMRL~\citep{li2021active}, and EDDI~\citep{ma2018eddi}. Unlike \ours, these methods use downstream labels during policy learning.

\begin{table}[h]
\centering
\setlength{\aboverulesep}{0pt}
\setlength{\belowrulesep}{0pt}
\caption{\small \textbf{Comparison with supervised acquisition on \minib.} Accuracy (\%) at different acquisition budgets. All methods except \ours use downstream labels. \legendbox{blue!40} denotes best and \legendbox{pink!40} second best performance.}
\label{tab:supervised_full}
\begin{tabular}{lcccc}
\toprule
\textbf{Method} & \textbf{10\%} & \textbf{20\%} & \textbf{40\%} & \textbf{60\%} \\ \midrule
DIME & $79.97$ & $84.23$ & $86.87$ & $87.07$ \\
GDFS & \best{$80.23$} & \best{$84.73$} & \secbest{$87.50$} & \secbest{$87.80$} \\
SEFA & $78.67$ & \secbest{$84.70$} & \best{$87.73$} & \best{$88.13$} \\
GSMRL & $75.73$ & $83.27$ & $87.37$ & $87.77$ \\
EDDI & $72.30$ & $76.70$ & $83.50$ & $85.07$ \\ \arrayrulecolor{lightgray} \hline \arrayrulecolor{black}
\ours & \secbest{$80.03$} & $81.70$ & $86.87$ & $87.60$ \\
\bottomrule
\end{tabular}
\end{table}

At 10\% acquisition, \ours is within 0.2 points of the best supervised method and exceeds four of the five, despite never observing downstream labels. The gap is largest at 20\%, where GDFS leads by 3.0 points, and narrows to 0.9 and 0.5 points at 40\% and 60\%, respectively. Given that \ours never observes downstream labels during training, we view this as the expected cost of task agnosticism.

\subsection{Cost of Task Agnosticism}
\label{appendix:task_specific}
When a single downstream task is fixed and sufficient labels are available, a task-specific policy may be preferable. To quantify this trade-off, we train supervised acquisition policies for NSCLC subtyping (S) and grading (G) using the same architecture and training budget as \ours, replacing the latent reward with the downstream-loss improvement
\begin{equation*}
r_t^{\mathrm{sup}} = \ell\big(g(\zb(\Pb_{t-1})), y\big) - \ell\big(g(\zb(\Pb_t)), y\big).
\end{equation*}
PPO-Subtype and PPO-Grade optimize their respective endpoints, while PPO-Multitask averages the two task-specific losses at each step. We evaluate each policy on both tasks.

\begin{table}[h]
\centering
\setlength{\aboverulesep}{0pt}
\setlength{\belowrulesep}{0pt}
\caption{\small \textbf{Task-specific and task-agnostic acquisition on NSCLC.} Accuracy (\%) for subtyping (S), grading (G), and their mean at different acquisition budgets. Supervised PPO policies use task labels during policy learning, while \ours uses only the representation-recovery reward.}
\label{tab:task_specific}
\adjustbox{max width=\textwidth}{
\begin{tabular}{lccccccccc}
\toprule
\multirow{2}{*}{\textbf{Policy}} & \multicolumn{3}{c}{\textbf{20\%}} & \multicolumn{3}{c}{\textbf{40\%}} & \multicolumn{3}{c}{\textbf{60\%}} \\
\arrayrulecolor{lightgray} \cmidrule(lr){2-4} \cmidrule(lr){5-7} \cmidrule(lr){8-10} \arrayrulecolor{black}
 & \textbf{S} & \textbf{G} & \textbf{Mean} & \textbf{S} & \textbf{G} & \textbf{Mean} & \textbf{S} & \textbf{G} & \textbf{Mean} \\ \midrule
PPO-Subtype & $83.2$ & $62.8$ & $73.0$ & $85.4$ & $64.2$ & $74.8$ & $86.3$ & $65.0$ & $75.7$ \\
PPO-Grade & $78.5$ & $67.8$ & $73.1$ & $82.2$ & $68.2$ & $75.2$ & $83.1$ & $68.0$ & $75.5$ \\
PPO-Multitask & $82.9$ & $66.5$ & $74.7$ & $84.1$ & $67.2$ & $75.6$ & $84.8$ & $67.9$ & $76.3$ \\ \arrayrulecolor{lightgray} \hline \arrayrulecolor{black}
\ours & $82.8$ & $65.8$ & $74.3$ & $83.8$ & $67.1$ & $75.5$ & $85.2$ & $67.4$ & $76.3$ \\
\bottomrule
\end{tabular}
}
\end{table}

At 20\% acquisition, \ours trails the corresponding task-specific policies by 0.4 points on subtyping and 2.0 points on grading (Table~\ref{tab:task_specific}). The specialized policies also transfer poorly across endpoints, losing approximately 3--5 points when evaluated on the other task. More importantly, \ours is within 0.4, 0.1, and 0.0 points of PPO-Multitask in mean accuracy at 20\%, 40\%, and 60\%, respectively, despite using neither task label. This quantifies the expected cost of task agnosticism, while showing that preserving the pretrained representation nearly matches direct multitask supervision.

\subsection{Robustness to Partially Observed Data}
Table~\ref{tab:extended-partial} provides additional robustness results under partially observed reference data on MNIST-14 and FMNIST-14. We train \ours after randomly masking an $m$-fraction of views when computing the reference representation used for the self-supervised reward. Across both datasets, performance degrades smoothly as m increases, rather than collapsing. On MNIST-14, partially observed-reference policies remain close to or even slightly above the complete-reference policy across several budgets, indicating that the acquisition policy can tolerate substantial noise in the reference target. On FMNIST-14, the degradation is more visible but remains gradual: even with 50\% missing reference views, \ours stays competitive with strong baselines and preserves most of its early-budget advantage. These results complement the biological experiments in Table~\ref{tab:incomplete_data}, showing that the method’s robustness to incomplete reference data is not specific to a single modality or dataset.

\subsection{Heterogeneous Acquisition Costs}
To test whether \ours can adapt when views have nonuniform acquisition costs, we run a controlled cost-aware experiment on MNIST-14. We assign each pixel a sample-dependent acquisition cost based on its intensity, rescaled to the range $[0.25,2.0]$, so that the policy must trade off representation recovery against cumulative cost. We compare a cost-aware variant of \ours, trained with a cost-penalized reward and cost information in the state, against the standard cost-unaware policy.

Fig.~\ref{fig:cost} shows that cost-aware training changes the acquisition behavior. Over the first ten acquisitions, the cost-aware policy selects substantially cheaper pixels on average than the cost-unaware policy, reducing the mean acquisition cost from 0.89 to 0.56. This shift also improves performance in the low-budget regime: when accuracy is plotted against cumulative cost rather than number of pixels, the cost-aware policy reaches high accuracy with fewer cost units. As the budget grows, both policies converge, indicating that cost-awareness is most useful when acquisition resources are limited. This controlled experiment confirms that \ours can incorporate heterogeneous costs by changing the reward and state, without modifying the representation-recovery objective.

\begin{figure}[t]
    \centering
    \includegraphics[width=\linewidth]{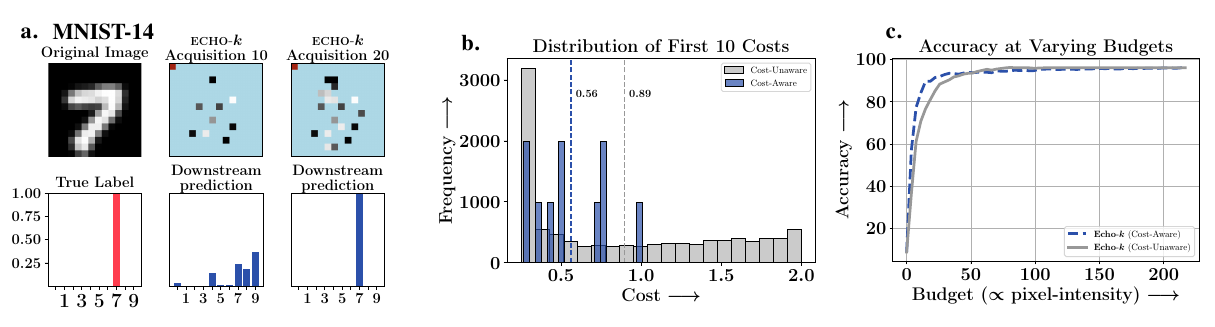}
    \caption{\textbf{a.} Overview of the task and acquisition behavior of \textsc{echo}-$k$: we illustrate how \ours\ sequentially acquires pixels in an unsupervised manner, progressively concentrating the (unobserved) downstream class probability mass. \textbf{b.} Distribution of acquisition costs among the first ten selected pixels; the cost-aware policy selects lower-cost pixels early. \textbf{c}. Accuracy as a function of cumulative acquisition cost; cost-aware training improves performance in the low-budget regime, while both policies converge at larger budgets.}
    \label{fig:cost}
\end{figure}

\begin{table}[t]
\centering
\caption{\textbf{
} We train \ours after randomly masking an $m$-fraction of views when computing the reference representation. Accuracy is reported at different acquisition budgets on MNIST-14 and FMNIST-14. \legendbox{blue!40} denotes best and \legendbox{pink!40} second best performance, apart from the Oracle.}
\label{tab:extended-partial}
\adjustbox{max width=\textwidth}{
\begin{tabular}{lrrrrrrrrr}
\toprule
\multirow{2}{*}{\textbf{Method}} & \multicolumn{4}{c}{\textbf{MNIST-14}} &  & \multicolumn{4}{c}{\textbf{FMNIST-14}} \\
\cmidrule{2-5} \cmidrule{7-10}
 & \textbf{20\%} & \textbf{40\%} & \textbf{60\%} & \textbf{80\%} &  & \textbf{20\%} & \textbf{40\%} & \textbf{60\%} & \textbf{80\%} \\ \midrule
Oracle & $0.937 {\scriptstyle \pm 0.005}$ & $0.951 {\scriptstyle \pm 0.002}$ & $0.953 {\scriptstyle \pm 0.002}$ & $0.956 {\scriptstyle \pm 0.002}$ &  & $0.824 {\scriptstyle \pm 0.003}$ & $0.843 {\scriptstyle \pm 0.006}$ & $0.840 {\scriptstyle \pm 0.003}$ & $0.842 {\scriptstyle \pm 0.004}$ \\ \arrayrulecolor{lightgray} \hline 
Variance & $0.759 {\scriptstyle \pm 0.007}$ & $0.865 {\scriptstyle \pm 0.005}$ & $0.928 {\scriptstyle \pm 0.005}$ & \best{$0.951 {\scriptstyle \pm 0.006}$} &  & $0.565 {\scriptstyle \pm 0.022}$ & $0.670 {\scriptstyle \pm 0.008}$ & $0.725 {\scriptstyle \pm 0.016}$ & $0.741 {\scriptstyle \pm 0.005}$ \\
\textsc{spec}-$\varphi_3$ & $0.168 {\scriptstyle \pm 0.014}$ & $0.754 {\scriptstyle \pm 0.018}$ & $0.894 {\scriptstyle \pm 0.003}$ & $0.942 {\scriptstyle \pm 0.002}$ &  & $0.543 {\scriptstyle \pm 0.036}$ & $0.600 {\scriptstyle \pm 0.039}$ & $0.681 {\scriptstyle \pm 0.012}$ & $0.738 {\scriptstyle \pm 0.013}$ \\
MCFS & $0.678 {\scriptstyle \pm 0.010}$ & $0.829 {\scriptstyle \pm 0.005}$ & $0.904 {\scriptstyle \pm 0.004}$ & $0.918 {\scriptstyle \pm 0.002}$ &  & $0.631 {\scriptstyle \pm 0.029}$ & $0.760 {\scriptstyle \pm 0.016}$ & $0.767 {\scriptstyle \pm 0.013}$ & $0.793 {\scriptstyle \pm 0.010}$ \\
Random & $0.627 {\scriptstyle \pm 0.009}$ & $0.800 {\scriptstyle \pm 0.005}$ & $0.871 {\scriptstyle \pm 0.003}$ & $0.907 {\scriptstyle \pm 0.003}$ &  & $0.708 {\scriptstyle \pm 0.004}$ & $0.779 {\scriptstyle \pm 0.004}$ & $0.805 {\scriptstyle \pm 0.002}$ & \secbest{$0.821 {\scriptstyle \pm 0.005}$} \\ 
\ours & $0.830 {\scriptstyle \pm 0.004}$ & \secbest{$0.913 {\scriptstyle \pm 0.012}$} & $0.935 {\scriptstyle \pm 0.009}$ & \best{$0.951 {\scriptstyle \pm 0.003}$} &  & \best{$0.749 {\scriptstyle \pm 0.019}$} & \best{$0.793 {\scriptstyle \pm 0.007}$} & \best{$0.818 {\scriptstyle \pm 0.007}$} & \best{$0.833 {\scriptstyle \pm 0.009}$} \\ \arrayrulecolor{lightgray} \hline 
\ours ($m=$10\%) & \secbest{$0.845 {\scriptstyle \pm 0.015}$} & \best{$0.921 {\scriptstyle \pm 0.002}$} & \best{$0.941 {\scriptstyle \pm 0.002}$} & $0.949 {\scriptstyle \pm 0.003}$ &  & \secbest{$0.742 {\scriptstyle \pm 0.024}$} & \secbest{$0.780 {\scriptstyle \pm 0.008}$} & \secbest{$0.814 {\scriptstyle \pm 0.008}$} & \secbest{$0.821 {\scriptstyle \pm 0.011}$} \\
\ours ($m=$20\%) & \best{$0.862 {\scriptstyle \pm 0.002}$} & \secbest{$0.913 {\scriptstyle \pm 0.006}$} & $0.939 {\scriptstyle \pm 0.007}$ & $0.948 {\scriptstyle \pm 0.003}$ &  & $0.735 {\scriptstyle \pm 0.020}$ & $0.767 {\scriptstyle \pm 0.013}$ & $0.803 {\scriptstyle \pm 0.006}$ & $0.813 {\scriptstyle \pm 0.008}$ \\
\ours ($m=$30\%) & $0.836 {\scriptstyle \pm 0.008}$ & $0.912 {\scriptstyle \pm 0.003}$ & \secbest{$0.940 {\scriptstyle \pm 0.008}$} & \secbest{$0.950 {\scriptstyle \pm 0.005}$} &  & $0.721 {\scriptstyle \pm 0.018}$ & $0.769 {\scriptstyle \pm 0.014}$ & $0.798 {\scriptstyle \pm 0.009}$ & $0.814 {\scriptstyle \pm 0.008}$ \\
\ours ($m=$40\%) & $0.830 {\scriptstyle \pm 0.017}$ & $0.903 {\scriptstyle \pm 0.014}$ & $0.934 {\scriptstyle \pm 0.001}$ & \secbest{$0.950 {\scriptstyle \pm 0.011}$} &  & $0.716 {\scriptstyle \pm 0.024}$ & $0.776 {\scriptstyle \pm 0.003}$ & $0.787 {\scriptstyle \pm 0.003}$ & $0.801 {\scriptstyle \pm 0.009}$ \\
\ours ($m=$50\%) & $0.823 {\scriptstyle \pm 0.019}$ & $0.900 {\scriptstyle \pm 0.011}$ & $0.938 {\scriptstyle \pm 0.007}$ & $0.949 {\scriptstyle \pm 0.002}$ &  & $0.711 {\scriptstyle \pm 0.013}$ & $0.765 {\scriptstyle \pm 0.018}$ & $0.792 {\scriptstyle \pm 0.006}$ & $0.815 {\scriptstyle \pm 0.002}$ \\
\bottomrule
\end{tabular}
}
\end{table}

\subsection{Adaptive Stopping}
\label{appendix:stopping}
A fixed acquisition horizon may be restrictive in deployment, where the policy should also decide whether the currently acquired views are already sufficient. We extend \ours with a STOP action, converting the fixed-horizon MDP into an optimal-stopping problem. For a measurement action $a$, the reward is
\begin{equation*}
r_t = \frac{e(\Pb_{t-1}) - e(\Pb_t)}{e(\Pb_0) + \epsilon} - \lambda c_a, \qquad e(\Pb) = \|\zb(\Pb) - \zb^{\star}\|_2,
\end{equation*}
while STOP terminates the episode with zero additional reward. Since STOP yields zero future reward, the policy continues only when the expected cumulative representation gain outweighs the acquisition cost, and with $\gamma = 1$, the learned stopping time $\tau$ directly trades recovery against total acquisition cost. As in \ours, $\zb^{\star}$ is used only to construct the label-free training reward; at deployment, both acquisition and stopping depend only on the current partial state.

\begin{wraptable}{r}{0.45\textwidth}
\centering
\setlength{\aboverulesep}{0pt}
\setlength{\belowrulesep}{0pt}
\caption{\small \textbf{Adaptive stopping on \minib.} Accuracy (\%) and number of acquired features.}
\label{tab:stopping}
\adjustbox{max width=\linewidth}{
\begin{tabular}{lcc}
\toprule
\textbf{Policy} & \textbf{\# Features} & \textbf{Accuracy} \\ \midrule
\ours (fixed horizon) & $10$ & $86.11$ \\
\ours (STOP) & $9.86$ (mean) & $86.06$ \\
\bottomrule
\end{tabular}
}
\end{wraptable}
On \minib, \ours with stopping reaches 86.06\% accuracy at a mean of 9.86 acquired features, showing instance-dependent stopping (Table~\ref{tab:stopping}). This matches the fixed-horizon policy at a comparable budget (86.11\% at $t=10$), while requiring no budget to be specified in advance. Together with the heterogeneous-cost experiment above, this shows that the representation-recovery objective of \ours can be combined with cost-aware rewards to decide both which views to acquire and when to stop.

\subsection{Empirical Validation of the Near Isometry Assumption}
\label{appendix:near_isometry}

A key ingredient in our theoretical analysis is Assumption~\ref{asm:isometry} (Near Isometry), which requires that the normalized columns of the embedding weight matrix are approximately orthogonal. We validate this assumption empirically on two classes of models.

\subsubsection{DinoV2 Foundation Models}

We examine the DINOv2 family of vision transformers~\citep{oquab2023dinov2}, including ViT-S/14, ViT-B/14, ViT-L/14, and ViT-G/14 (with register tokens). For each model, we: (1) extract the final MLP projection weights $\Wb \in \RR^{p \times d}$ from the last transformer block's feed-forward network (fc2 layer); (2) normalize columns to get $\bm{B}$ and compute the Gram matrix $\bm{G} = \bm{B}^\top \bm{B}$; (3) analyze the distribution of the off-diagonal elements of $\bm{G} - \bm{I}_d$. As shown in Fig.~\ref{fig:app-weight-dinov2}, these values concentrate tightly near zero across all model scales, confirming approximate orthogonality. Larger models exhibit tighter concentration, consistent with theoretical expectations for high-dimensional random-like weight matrices.

\begin{figure}[H]
    \centering
    \includegraphics[width=0.8\linewidth,page=1]{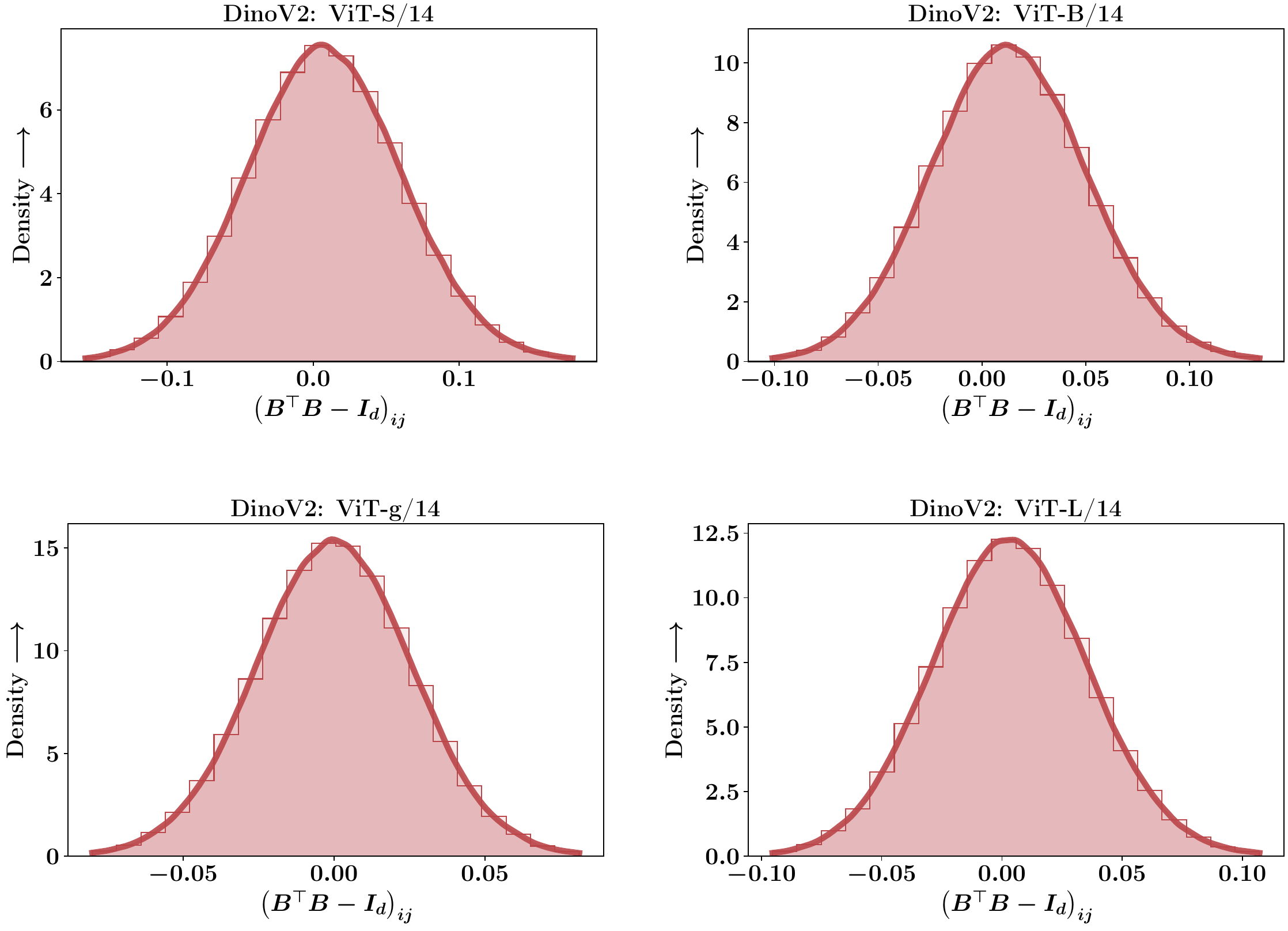}
    \caption{Distribution of $\bm{B}^\top\bm{B}-\bm{I}_d$ for 4 different DinoV2 models of increasing capacity.}
    \label{fig:app-weight-dinov2}
\end{figure}

\subsubsection{Trained Autoencoders}
\label{appendix:viz_weights}

We further validate the assumption on trained autoencoders across varying dimensions $d$ from 32 to 4096. For each dimension, we generate low-rank data $\bm{X}^\top = \bm{U} \bm{V}^\top + \bm{\epsilon}$ where $\bm{U} \in \RR^{n \times r}$, $\bm{V} \in \RR^{d \times r}$. We extract the last layer weight matrix $\Wb$ and compute the effective rank $\text{p}_{eff}$ using the entropy of normalized singular values~\citep{roy2007effective}:
\[
p_{\text{eff}} = \exp\left(-\sum_i \frac{s_i}{\sum_k s_k} \log \frac{s_i}{\sum_k s_k}\right).
\]
We sample random subsets $J \subseteq \{1, \dots, d\}$ of varying sizes $m$, normalize the corresponding weight vectors to form $\bm{B}_J$, and compute $\|\bm{B}_J^\top \bm{B}_J - \bm{I}_m\|_{\text{op}}$. We compare against the theoretical bound $C\left(\sqrt{m/p_{\text{eff}}} + m/p_{\text{eff}}\right)$. Across 50 repetitions, the near isometry property holds with $C \in [0.93,1.73]$ across all tested dimensions. As $d$ increases, the standard deviation of off-diagonal elements decreases, validating that trained neural network weights naturally form near-orthonormal systems (Fig.~\ref{fig:app-weight-syn}).

\begin{figure}[t]
    \centering
    \includegraphics[width=0.8\linewidth,page=2]{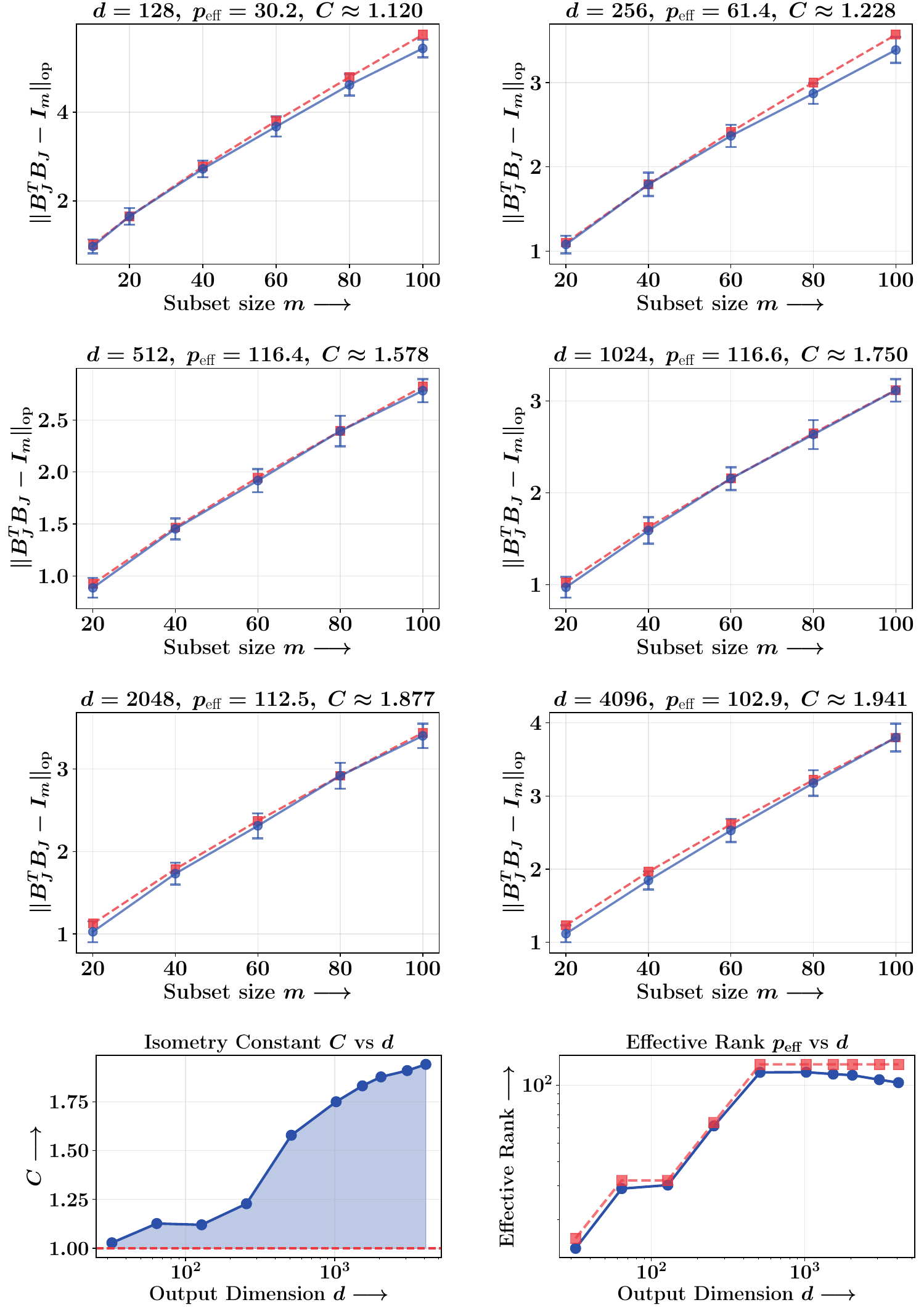}
    \caption{Variation of $\|\bm{B}^\top\bm{B}-\bm{I}_d\|_{\text{op}}$ for varying values of $d$ on trained autoencoder models and variation of the empirical estimation of $C$ with $d$.}
    \label{fig:app-weight-syn}
\end{figure}
\begin{figure}[H]
    \centering
    \includegraphics[width=0.8\linewidth]{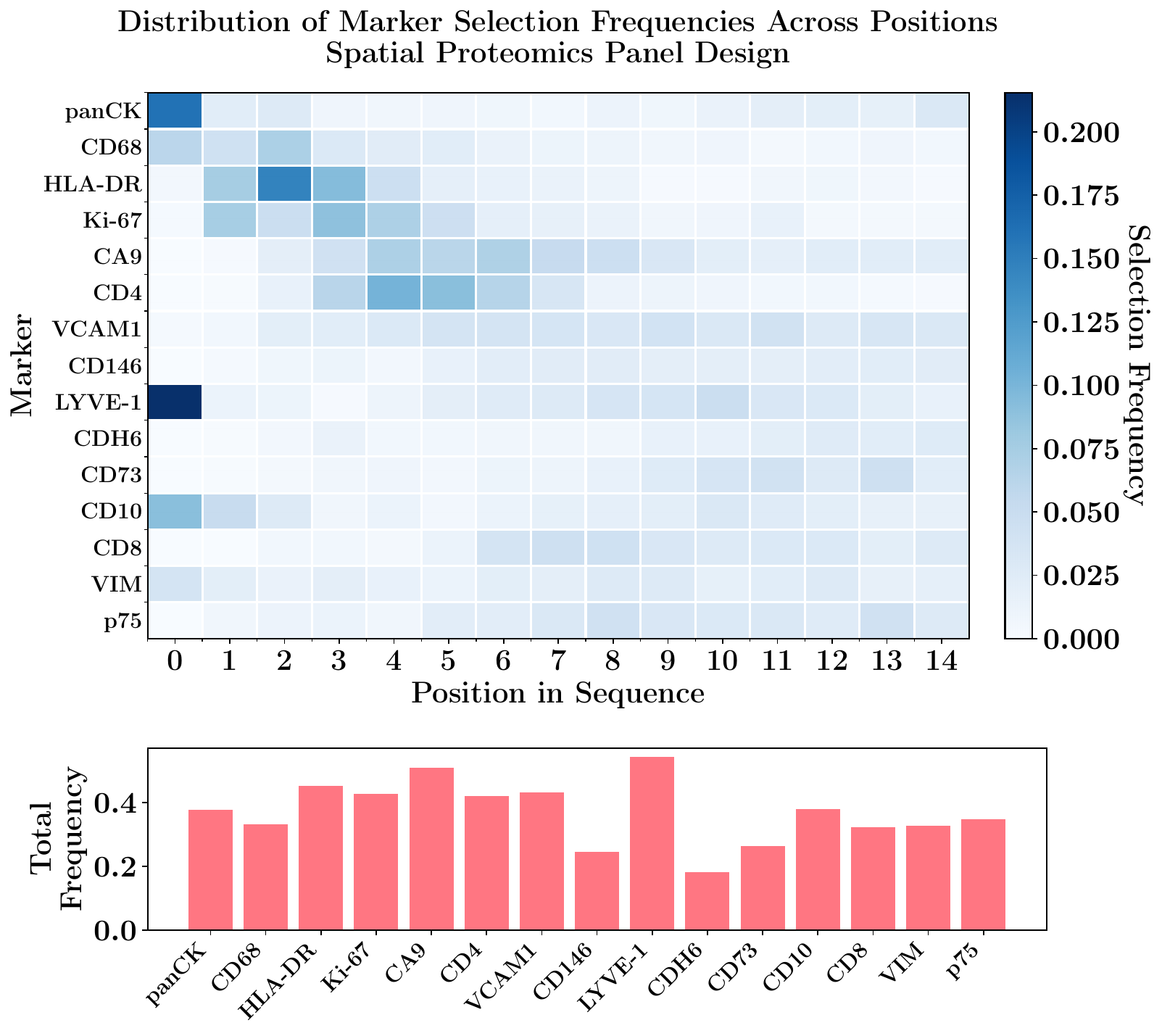}
    \caption{Analysis of the top-15 markers and their frequency, as chosen by \ours jointly for the cancer subtyping and grading task on the NSCLC spatial proteomics dataset.}
    \label{fig:app-weight-vir}
\end{figure}
\subsection{Analysis of Markers Acquired by \ours on Spatial Proteomics}
\label{appendix:marker_analysis}

In Fig.~\ref{fig:app-weight-vir}, we analyze the distribution of markers chosen by \ours on the NSCLC dataset tasks: cancer subtyping and grading from the marker panel provided by~\citep{cords2024cancer}. In an unsupervised fashion, \ours prioritizes markers that are well-established in the clinical literature for their roles in tumor characterization and microenvironmental context:

\begin{itemize}[leftmargin=1.5em]
    \item \texttt{panCK}: a pan-cytokeratin marker for epithelial (tumor) cell identity, consistently selected first.
    \item \texttt{Ki-67}: a proliferation marker widely used in clinical grading of lung cancers~\citep{martin2004ki}.
    \item \texttt{CD4}: a marker of helper T-cell infiltration, reflecting the adaptive immune response in the tumor microenvironment.
    \item \texttt{LYVE-1}: a marker of lymphatic endothelial cells, relevant to lymphatic involvement and metastatic potential.
\end{itemize}

This early selection pattern suggests that \ours effectively identifies markers with strong relevance to both tumor biology and immune context, thereby supporting both subtyping and grading tasks without any supervision. The alignment with clinical domain knowledge provides additional validation that the self-supervised acquisition objective captures diagnostically meaningful information.

\newpage
\section*{NeurIPS Paper Checklist}

The checklist is designed to encourage best practices for responsible machine learning research, addressing issues of reproducibility, transparency, research ethics, and societal impact. Do not remove the checklist: {\bf The papers not including the checklist will be desk rejected.} The checklist should follow the references and follow the (optional) supplemental material.  The checklist does NOT count towards the page
limit. 

Please read the checklist guidelines carefully for information on how to answer these questions. For each question in the checklist:
\begin{itemize}
    \item You should answer \answerYes{}, \answerNo{}, or \answerNA{}.
    \item \answerNA{} means either that the question is Not Applicable for that particular paper or the relevant information is Not Available.
    \item Please provide a short (1--2 sentence) justification right after your answer (even for \answerNA). 
\end{itemize}

{\bf The checklist answers are an integral part of your paper submission.} They are visible to the reviewers, area chairs, senior area chairs, and ethics reviewers. You will also be asked to include it (after eventual revisions) with the final version of your paper, and its final version will be published with the paper.

The reviewers of your paper will be asked to use the checklist as one of the factors in their evaluation. While \answerYes{} is generally preferable to \answerNo{}, it is perfectly acceptable to answer \answerNo{} provided a proper justification is given (e.g., error bars are not reported because it would be too computationally expensive'' or ``we were unable to find the license for the dataset we used''). In general, answering \answerNo{} or \answerNA{} is not grounds for rejection. While the questions are phrased in a binary way, we acknowledge that the true answer is often more nuanced, so please just use your best judgment and write a justification to elaborate. All supporting evidence can appear either in the main paper or the supplemental material, provided in appendix. If you answer \answerYes{} to a question, in the justification please point to the section(s) where related material for the question can be found.

IMPORTANT, please:
\begin{itemize}
    \item {\bf Delete this instruction block, but keep the section heading ``NeurIPS Paper Checklist"},
    \item  {\bf Keep the checklist subsection headings, questions/answers and guidelines below.}
    \item {\bf Do not modify the questions and only use the provided macros for your answers}.
\end{itemize}


\begin{enumerate}

\item {\bf Claims}
    \item[] Question: Do the main claims made in the abstract and introduction accurately reflect the paper's contributions and scope?
    \item[] Answer: \answerYes{} 
    \item[] Justification: The paper introduces \textsc{ECHO}-$k$, provides theoretical justifications, and experiments, as reflected in the abstract and introduction.

\item {\bf Limitations}
    \item[] Question: Does the paper discuss the limitations of the work performed by the authors?
    \item[] Answer: \answerYes{} 
    \item[] Justification: The paper contains an impact and limitations statement in the appendix which discusses the limitations of the work.

\item {\bf Theory assumptions and proofs}
    \item[] Question: For each theoretical result, does the paper provide the full set of assumptions and a complete (and correct) proof?
    \item[] Answer: \answerYes{} 
    \item[] Justification: The paper provides a central theorem bridging the oracle-test-time-policy gap and assumptions in the main, along with detailed proof in the Appendix.

    \item {\bf Experimental result reproducibility}
    \item[] Question: Does the paper fully disclose all the information needed to reproduce the main experimental results of the paper to the extent that it affects the main claims and/or conclusions of the paper (regardless of whether the code and data are provided or not)?
    \item[] Answer: \answerYes{} 
    \item[] Justification: The paper provides all experimental details, evaluation protocols, hyperparameters, hardware architecture, and Python libraries.

\item {\bf Open access to data and code}
    \item[] Question: Does the paper provide open access to the data and code, with sufficient instructions to faithfully reproduce the main experimental results, as described in supplemental material?
    \item[] Answer: \answerYes{} 
    \item[] Justification: The data /models are publicly available, and we attach the code to train and evaluate as a zip file in the supplement. 

\item {\bf Experimental setting/details}
    \item[] Question: Does the paper specify all the training and test details (e.g., data splits, hyperparameters, how they were chosen, type of optimizer) necessary to understand the results?
    \item[] Answer: \answerYes{} 
    \item[] Justification: All experimental details are provided in the main and appendix.

\item {\bf Experiment statistical significance}
    \item[] Question: Does the paper report error bars suitably and correctly defined or other appropriate information about the statistical significance of the experiments?
    \item[] Answer: \answerYes{} 
    \item[] Justification: All experiments have been run multiple times for statistical significance, and the error bars have been provided for all experiments.

\item {\bf Experiments compute resources}
    \item[] Question: For each experiment, does the paper provide sufficient information on the computer resources (type of compute workers, memory, time of execution) needed to reproduce the experiments?
    \item[] Answer: \answerYes{} 
    \item[] Justification: All hardware details for the experiments have been provided in the appendix.
    
\item {\bf Code of ethics}
    \item[] Question: Does the research conducted in the paper conform, in every respect, with the NeurIPS Code of Ethics \url{https://neurips.cc/public/EthicsGuidelines}?
    \item[] Answer: \answerYes{} 
    \item[] Justification: Yes, the research conducted conforms with the Code of Ethics. 

\item {\bf Broader impacts}
    \item[] Question: Does the paper discuss both potential positive societal impacts and negative societal impacts of the work performed?
    \item[] Answer: \answerYes{} 
    \item[] Justification: The paper contains an impact and limitations statement which discusses the impact of the work. 
    
\item {\bf Safeguards}
    \item[] Question: Does the paper describe safeguards that have been put in place for responsible release of data or models that have a high risk for misuse (e.g., pre-trained language models, image generators, or scraped datasets)?
    \item[] Answer: \answerNA{} 
    \item[] Justification: The work uses standard datasets published in other works, and does not use generative models.

\item {\bf Licenses for existing assets}
    \item[] Question: Are the creators or original owners of assets (e.g., code, data, models), used in the paper, properly credited and are the license and terms of use explicitly mentioned and properly respected?
    \item[] Answer: \answerYes{} 
    \item[] Justification: All original creators of code, data, and models are cited thoroughly in the paper.

\item {\bf New assets}
    \item[] Question: Are new assets introduced in the paper well documented and is the documentation provided alongside the assets?
    \item[] Answer: \answerNA{} 
    \item[] Justification: No new assets are introduced in the paper.

\item {\bf Crowdsourcing and research with human subjects}
    \item[] Question: For crowdsourcing experiments and research with human subjects, does the paper include the full text of instructions given to participants and screenshots, if applicable, as well as details about compensation (if any)? 
    \item[] Answer: \answerNA{} 
    \item[] Justification: The paper does not contain crowdsourcing experiments.

\item {\bf Institutional review board (IRB) approvals or equivalent for research with human subjects}
    \item[] Question: Does the paper describe potential risks incurred by study participants, whether such risks were disclosed to the subjects, and whether Institutional Review Board (IRB) approvals (or an equivalent approval/review based on the requirements of your country or institution) were obtained?
    \item[] Answer: \answerNA{} 
    \item[] Justification: The paper does not contain research with human subjects.

\item {\bf Declaration of LLM usage}
    \item[] Question: Does the paper describe the usage of LLMs if it is an important, original, or non-standard component of the core methods in this research? Note that if the LLM is used only for writing, editing, or formatting purposes and does \emph{not} impact the core methodology, scientific rigor, or originality of the research, declaration is not required.
    \item[] Answer: \answerNA{} 
    \item[] Justification: LLMs were only used for editing of grammar in the paper.

\end{enumerate}

\end{document}